\documentclass{article}
\usepackage[a4paper, total={6in, 10in}]{geometry}

\usepackage{amsmath,amsfonts,bm}

\def\eqref#1{equation~\ref{#1}}

\def\1{\bm{1}}

\DeclareMathAlphabet{\mathsfit}{\encodingdefault}{\sfdefault}{m}{sl}
\SetMathAlphabet{\mathsfit}{bold}{\encodingdefault}{\sfdefault}{bx}{n}

\usepackage[colorlinks,
            linkcolor=red,
            anchorcolor=blue,
            citecolor=blue
            ]{hyperref}
\usepackage{natbib}
\usepackage[capitalize,noabbrev]{cleveref}

\usepackage{url}
\usepackage{graphicx}
\usepackage{amsmath}
\usepackage{amssymb}
\usepackage{mathtools}
\usepackage{amsthm}
\usepackage{tcolorbox}
\usepackage{blkarray}
\usepackage{nicematrix}
\usepackage{subcaption}
\usepackage{enumitem}
\usepackage{booktabs}
\usepackage{float}
\makeatletter
\newcommand{\subalign}[1]{%
  \vcenter{%
    \Let@ \restore@math@cr \default@tag
    \baselineskip\fontdimen10 \scriptfont\tw@
    \advance\baselineskip\fontdimen12 \scriptfont\tw@
    \lineskip\thr@@\fontdimen8 \scriptfont\thr@@
    \lineskiplimit\lineskip
    \ialign{\hfil$\m@th\scriptstyle##$&$\m@th\scriptstyle{}##$\hfil\crcr
      #1\crcr
    }%
  }%
}
\makeatother

\title{What Can Transformer Learn with Varying Depth? Case Studies on Sequence Learning Tasks}

\author{
Xingwu Chen\thanks{Department of Computer Science, The University of Hong Kong. Email: \texttt{xingwu@connect.hku.hk}} \and Difan Zou\thanks{Department of Computer Science and Insititute of Data Science, The University of Hong Kong. Email: \texttt{dzou@cs.hku.hk}}
}

\date{}
\theoremstyle{plain}
\newtheorem{theorem}{Theorem}[section]
\newtheorem{proposition}[theorem]{Proposition}
\newtheorem{lemma}[theorem]{Lemma}

\theoremstyle{definition}
\newtheorem{definition}[theorem]{Definition}

\theoremstyle{remark}

\begin{document}

\maketitle

\begin{abstract}
We study the capabilities of the transformer architecture with varying depth. Specifically, we designed a novel set of sequence learning tasks to systematically evaluate and comprehend how the depth of transformer affects its ability to perform memorization, reasoning, generalization, and contextual generalization. We show a transformer with only one attention layer can excel in memorization but falls short in other tasks. 
Then, we show that exhibiting reasoning and generalization ability requires the transformer to have at least two attention layers, while context generalization ability may necessitate three attention layers. Additionally, we identify a class of simple operations that a single attention layer can execute, and show that the complex tasks can be approached as the combinations of these simple operations and thus can be resolved by stacking multiple attention layers. This sheds light on studying more practical and complex tasks beyond our design. Numerical experiments corroborate our theoretical findings.

\end{abstract}

\section{Introduction}

Transformers \citep{vaswani2017attention} have been recognized as the most powerful model to achieve state-of-the-art performances in various deep learning tasks such as vision, natural language process, and decision making \citep{dosovitskiy2020image, brown2020language,chen2021decision}. Its superior performance makes it the most prevalent architecture for building universal foundation models \citep{touvron2023llama,dosovitskiyImageWorth16x162021,devlin2018bert,ying2021transformers,ouyang2022training}, its superiority in capability makes it one of the most widely used architectures for building universal foundation models \citep{bommasani2021opportunities,brown2020language,kaplanScalingLawsNeural2020a}. A central module in the transformer is the attention layer, which performs nonlinear sequence-to-sequence mapping that allows each token to attend to several other tokens based on the semantic relationship. By stacking multiple attention layers, the transformer models have been observed to be surprisingly strong at performing memorization, understanding, and reasoning from the input sequences. 

The remarkable empirical performance of transformer has triggered a series of theoretical studies, which aim to understand the working mechanism of transformer. For instance, some early attempts, including RASP \citep{weissThinkingTransformers2021} and Tracr \citep{lindner2023tracr}, propose to interpret the transformer model by translating its mechanism into programming languages. However, their explanations are still hard to parse and difficult to help obtain quantitative characterizations on the transformer's capability. More recently, people has particularly focused on the capability of transformer in certain aspects, including its universal approximation power 
 \citep{kajitsukaAreTransformersOne2023,yunAreTransformersUniversal2020}, data memorization capacity \citep{kajitsukaAreTransformersOne2023,mahdaviMemorizationCapacityMultiHead2023}, reasoning ability \citep{boix-adseraWhenCanTransformers2023a, fu2023can}, and in-context learning (ICL) \citep{xie2021explanation,garg2022can,baiTransformersStatisticiansProvable2023,vonoswaldTransformersLearnIncontext2023}.

However, these research primarily focuses on specific, simplified tasks that only utilize a subset of the transformer's capabilities. 
In practice, tasks often involve complex combinations of these simpler tasks, rendering them more challenging. Moreover, these studies often assume that the data is well-structured, aligning perfectly with the desired input-output token pairs. In practical scenarios, transformer inputs typically consist of general sequences, with tokens generated through human learning processes. Consequently, it remains unclear whether a given transformer model can effectively handle practical sequence-based tasks that require leveraging multiple aspects of the transformer's capabilities. Specifically, the capacity and limitations of the transformer architecture for addressing diverse sequence learning tasks remain uncertain.

In this paper, we aim to comprehensively understand the performance of the attention-based transformer architecture by investigating whether and how certain tasks can be learned by transformers with varying depth (i.e. the number of attention layers). Specifically, we have designed four sequence learning tasks, including sequence classification, in-context question answering, template matching, and in-context template matching tasks, aiming at assessing and understanding the transformer's memorization, reasoning, generalization, and contextual generalization abilities. Notably, these tasks are correlated and purposely designed to incrementally become harder, based on which we can analyze how these abilities vary depending on the number of attention layers employed and characterize the mechanism of different attention layers. We have then conducted a systematic theoretical analysis to address two key research questions: (1) the minimum number of attention layers required for the transformer to perform the four tasks; and (2) the respective roles of the different attention layers in accomplishing these tasks. Our contributions to the field are summarized as follows:

\begin{itemize}
    \item We propose a new set of sequence learning tasks specifically designed to assess the capabilities of transformers. In contrast to prior research that often concentrates on isolated tasks with well-structured input data, our tasks are systematic, interconnected, and more representative of real-world scenarios (the input data are general sequences generated from human's learning process). By leveraging these tasks, we can accurately evaluate the transformer's proficiency in key areas such as memorization, reasoning, generalization, and contextual generalization, and interpret the underlying mechanism of attention layers.

    \item We then theoretically assess the learning ability of transformer with varying numbers of attention layers by presenting both positive and negative results. In particular, we prove that the transformer with single attention layer can memorize but fails on other tasks. On the opposite, we show that two-layer transformer can successfully perform the reasoning and generalization tasks, and the transformer may need $3$ layers to conduct contextual generalization. We further conduct numerical experiments to validate the theoretical results. These theoretical findings justify the need of more attention layers to accomplish more complicated tasks (that require multi-step reasoning and generalization), which aligns with the emergence phenomenon of transformer \citep{weiEmergentAbilitiesLarge2022}. 

    \item We further provide some evidences regarding the working mechanism of transformer to accomplish the designed tasks. We show that the single attention layer can perform simple copying, parsing, matching, and mapping operations. Then stacking multiple attention layers can achieve the combinations of these operations, thus accomplish the harder tasks. In our experiments, we show that the attention maps of a trained transformer for different tasks are consistent with our findings. This could be of independent interest to understand how transformer tackle more complicated tasks in practice.

\end{itemize}

\section{Related Work}

\textbf{Theoretical Understanding of Transformers.}  Remarkable achievements of transformer leads to various theoretical attempts to understand its underlying mechanisms. These works approach the understanding of transformers from different angles. From a universal-approximation perspective, researchers have proven that transformers can approximate any sequence-to-sequence mapping under mild assumptions about the data distribution and target functions\citep{yunAreTransformersUniversal2020,kajitsukaAreTransformersOne2023,mahdaviMemorizationCapacityMultiHead2023,takakura2023approximation}. In addition to mapping sequences, there is a line of work that investigates the transformer's ability to learn in context \citep{vonoswaldTransformersLearnIncontext2023,gargWhatCanTransformers2023,guoHowTransformersLearn2023,zhangTrainedTransformersLearn2023}, generalize on certain tasks \citep{boix-adseraWhenCanTransformers2023a} and even perform complex instructions \citep{pmlr-v202-giannou23a, liu2022transformers}.  While these works provide useful perspectives on what transformers can do and propose possible mechanisms, they often involve more layers than what is typically used in practice or fall short in explaining real-world tasks involving discrete tokens and functions. Additionally, some works try to understand transformers from a computational perspective, offering valuable insights for understanding important properties such as chain of thought \citep{feng2024towards,merrill2023expresssive,li2024chain}. Although these works show the expressive power and limitations of well-structured transformers for certain tasks, the detailed analysis of expressive power in specific layers remains unclear.

\textbf{Empirical Understanding of Transformers.} In addition to theoretical investigations, researchers have also attempted to understand the mechanisms of transformers through empirical analysis,  such as interpreting trained transformers to derive human-readable representations \citep{lindner2023tracr,friedman2023learning,weissThinkingTransformers2021,zhou2023algorithms}, explaining transformers through probing techniques
\citep{clark2019does,prabhu2022adapting,zou2023representation} or leveraging other large language models \citep{bills2023language}. However, due to the complexity of large language models, the explanations derived from these experiments are often complex and challenging to comprehend. Moreover, these empirical methods only provide insight into \textit{how} the model accomplishes certain tasks, while the underlying mechanisms and the minimum requirements for transformers to learn such algorithms, such as the minimum number of layers and attention heads, remain elusive. In comparison to previous theoretical work, we introduce a practical setting that adapts discrete functions and data. Unlike using random features, we employ an approach that is more easily explainable. Furthermore, we aim to provide a theoretical explanation for why smaller models with fewer layers struggle with certain tasks, instead of relying solely on experimental results. To the best of our knowledge, this is the first study that compares and explains the limitations of small transformers.

\section{Preliminaries}
\label{sec:problem_set}
\begin{figure*}[t]\label{fig:tasks}
\begin{center}
  \includegraphics[width=\textwidth]{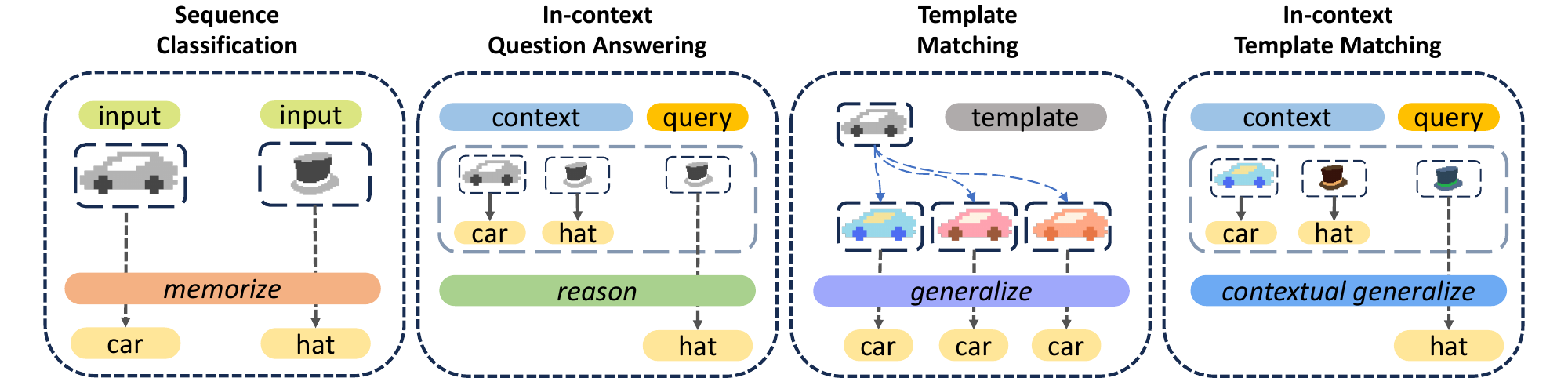}
  \caption{Descriptions of the four sequence learning tasks considered in this work, including (1) sequence classification task; (2) in-context question answering task; (3) template matching task; and (4) in-context template matching task. Here each input, context, and query are represented as sequences consisting of multiple tokens. }
\label{tasks}
\end{center}

\end{figure*}

\textbf{Notations.} The set of indices from $0$ to $n-1$ is denoted by $[n]$. Boldface upper-case $\mathbf{X}$ and lower-case $\mathbf{x}$ represent matrices and vectors, respectively. Specifically, we use $[\cdot]$ as Python index notation where $\mathbf{X}[i,:]$ refers to the $i$-th row of $\mathbf{X}$ and $\mathbf{X}[:,j]$ refers to the $j$-th column of $\mathbf{X}$. Similarly, $\mathbf{x}[i]$ refers to the $i$-th element of $\mathbf{x}$.

\subsection{Attention-only Transformers}
The transformer \citep{vaswani2017attention} is a neural network that can map a matrix $[\mathbf{x}_0,\dots,\mathbf{x}_{n-1}]$ of size $d \times n$ to a sequence $[\mathbf{y}_0,\dots,\mathbf{y}_{n-1}]$. In this work, we consider a transformer with $L$ hidden layers and a classifier output layer:
\begin{equation}\label{eq:transformer}
    \texttt{TF} = \underbrace{f_{\texttt{cls}}}_{\text{classifier}}\circ \underbrace{ \texttt{TF}_L \circ \dots \circ \texttt{TF}_1 }_{L \text{ hidden layers}}.
\end{equation}

Moreover, as mentioned previously, the objective of this work is to investigate the reasoning and generalization ability of the attention calculations in the transformer. Thus in each hidden layer, we choose to explode the MLP module as it performs token-wise operations that may introduce unnecessary distortion to our analysis \citep{zhang2017understanding}. Mathematically, given the representation matrix $\mathbf H^{(l)}\in\mathbb R^{d'\times n}$ in the $(l+1)$-th layer, where $d'$ denotes the dimension of the hidden representations, the transformer layer $\texttt{TF}_{l+1}$ with $m$ attention heads computes:
\begin{align}\label{eq:transformer_attn}
    \mathbf{H}^{(l+1)} &= \texttt{TF}_{l+1}(\mathbf H^{(l)})\notag\\ & = \mathbf{H}^{(l)} + \frac{1}{m}\sum_{i = 1}^{m}\big[\mathbf{V}^{(l)}_i \sigma\big((\mathbf{Q}^{(l)}_i)^\top \mathbf{K}^{(l)}_i\big)\big ],
\end{align}
where $\mathbf{Q}^{(l)}_i = \mathbf{W}^{(l)}_{Q_i}\mathbf{H}^{(l)} $, $\mathbf{K}^{(l)}_i=\mathbf{W}^{(l)}_{K_i}\mathbf{H}^{(l)}$, and $\mathbf{V}^{(l)}_i=\mathbf{W}^{(l)}_{V_i}\mathbf{H}^{(l)}$ are the query, key, and value computed by the $i$-th attention head with learned weight matrices $\mathbf{W}^{(l)}_{Q_i}$, $\mathbf{W}^{(l)}_{K_i}$, and $\mathbf{W}^{(l)}_{V_i}$ respectively. These weight matrices have dimensions $\mathbb{R}^{d' \times d'}$. Besides, $\sigma(z)$ is the activation function, which is set as the ReLU function $ \sigma(z)= \max \{0,z\}$ in this work. 

\noindent\textbf{Classifier Layer:} Given the representation generated by the last hidden layer $\texttt{TF}_L$, i.e, $\mathbf{H}^{(L)}=(\mathbf{h}^{(L)}_0,\dots,\mathbf{h}^{(L)}_{n-1})$, we make use of its last column, i.e., $\mathbf{h}^{(L)}_{n-1}$ to obtain the final prediction  $\mathbf{o} = \mathbf{W}_{O} \mathbf{h}^{(L)}_{n-1} \in \mathbb{R}^{C}$, where $\mathbf{W}_{O} \in \mathbb{R}^{C \times d'}$ is the weight matrix of the classifier layer. Notably, $C$ represents the total number of labels, which can be seen as (1) the vocabulary size for the token prediction task; or (2) the number of classes for the sequence classification task. The prediction result is then achieved by finding the index of the maximum entry of $\mathbf{o}$, i.e., $\hat y = \arg\max_{i\in[C]} \mathbf{o}[i]$.

\noindent\textbf{Positional Encoding and Padding}: Given a sequence of discrete tokens, denoted by $\mathbf{X} = [\mathbf{x}_0,\dots,\mathbf{x}_{n-1}] \in \mathbb{R}^{d \times n}$, the initial representation of each token is composed by the original token embedding, positional encoding, and padding. In particular, note that the hidden dimension is $d'$, the initial representation matrix for the sequence $\mathbf H^{(0)}$ is given by:
\begin{equation}\label{eq:H0}
    \mathbf{H}^{(0)} = \begin{NiceArray}{\left\lgroup ccc \right\rgroup l}
  \mathbf{x}_0 & \cdots & \mathbf{x}_{n-1} & \rightarrow d \times n \\
  \mathbf{p}_0 & \cdots & \mathbf{p}_{n-1} & \rightarrow n \times n \\
  \mathbf{0} & \cdots & \mathbf{0} & \rightarrow (d' - n - d) \times n \\
\end{NiceArray}.
\end{equation}
For the simplicity of analysis, we consider the one-hot positional encoding, i.e., we set $\mathbf{p}_i = [\mathbf{0}_{i},1,\mathbf{0}_{n -  (i+1)}]^\top$ for position $i$. 

\section{Memorization, Reasoning, and Generalization Tasks for Sequences}
\label{sec:tasks}

In this section, we will introduce the tasks designed to assess and understand the capability of transformers for tackling sequences. In particular, four tasks will be designed, which aim to characterize the capability of transformer structure in terms of memorization, reasoning, generalization, and contextual generalization.

\subsection{Memorization: Sequence Classification Task} 
The memorization capability serves as a fundamental theoretical property for transformers.  We start our understanding of the transformer model by characterizing its memorization capability. In particular, we consider the sequence classification task, as shown in \ref{tasks}, one of the most important and successful tasks for transformer-based models \citep{devlin2018bert}.
To formulate the sequence classification task,  we define the dataset $\mathcal{D}_{\texttt{SC}}$ as a collection of $N$ sequence-label examples, each with a different class type. Specifically, $\mathcal{D}_{\texttt{SC}} = \{ (\mathbf{X}^{(0)},y^{(0)}),\dots, (\mathbf{X}^{(N-1)},y^{(N-1)})\}$, where $\mathbf{X} \in \mathbb{R}^{d \times n}$ is a sequence consisting of $n$ discrete tokens from a word alphabet $\mathcal{X}$, and the corresponding labels ${y^{(0)},y^{(1)},\dots,y^{(N-1)}}$ are distinct integer. Before input into the model, we first append a \texttt{CLS} token $\mathbf{c}$ at the end of each sequence, which is widely applied in transformer-based models as a representation of the whole sequence. Then this task is to characterize whether the transformer model can successfully map each sequence to the corresponding label based on the representation corresponding to the last token of the sequence, i.e., $\texttt{CLS}$.

\subsection{Reasoning: In-context Question Answering Task} 
In-context learning \citep{brownLanguageModelsAre2020} refers to the capability of model to learn from the context and provide answers to questions based on examples and their corresponding solutions. To characterize the reasoning capability of the transformer, we consider a simplified in-context learning task, called in-context question answering task, which is summarized in \ref{tasks}. We consider a simple in-context learning problem with several question-answer pairs, the model is required to retrieve the corresponding answer based on the given question from the context. 

To formulate our in-context question-answering task, we define three types of tokens: \textit{question tokens} $\mathcal{Q} = \{\mathbf{q}_0, \mathbf{q}_1, \dots, \mathbf{q}_{n_q-1}\}$, \textit{response sign} $\mathcal{R} =  \{\mathbf{r}\}$, and \textit{answer tokens} $\mathcal{A} = \{\mathbf{a}_0, \mathbf{a}_1, \dots, \mathbf{a}_{n_a-1}\}$. Additionally, we use $\pi_0,\dots,\pi_{n-1}$ and $\pi_0',\dots,\pi_{n-1}'$ denote the indices of the sampled question and answer tokens, $\pi$ and $\pi'$ correspond to permutations. The \textit{response sign} is a special tag widely used in large language models like Llama2 \citep{touvron2023llama} and Galactica \citep{taylorGalacticaLargeLanguage2022a} for guiding the model's behavior in question-answer scenarios. Our data is constructed as follows: we sample $k$ questions from $\mathcal{Q}$, denoted as $(\mathbf{q}_{\pi_0}, \dots, \mathbf{q}_{\pi_{k-1}})$ and $k$ answers $(\mathbf{a}_{\pi'_0}, \dots, \mathbf{a}_{\pi'_{k-1}})$ from $\mathcal{A}$. We then add the response sign $\mathbf{r}$ between each $\mathbf{q}_{\pi_i}$ and  $\mathbf{a}_{\pi'_i}$, resulting in a context block length of $3k$: $\mathbf{B}_{\texttt{IC-QA}}^{(\pi, \pi')} = (\mathbf{q}_{\pi_0}, \mathbf{r}, \mathbf{a}_{\pi'_0}, \dots, \mathbf{q}_{\pi_{k-1}}, \mathbf{r}, \mathbf{a}_{\pi'_{k-1}})$. Next, we randomly choose the question $\mathbf{q}_{\pi_c}$ (where $c \in [k]$)  from the context block, and concatenate it with the final response sign: $[\mathbf{B}_{\texttt{IC-QA}}^{(\pi,\pi')}; \mathbf{q}_{\pi_c}; \mathbf{r}]$.  We denote this data as $\mathbf{E}_{\texttt{IC-QA}}^{(\pi,\pi')}(c)$, a sequence length $n = 3k+2$. Instead of pre-defining question-answer pairs, we consider each question to have $n_a$ possible answers by choosing different permutation $\pi$ and $\pi'$, as the objective is to investigate whether the model can learn to retrieve answers from the context, rather than memorizing the question-answer pairs. In this way, with $n_q$ questions and $n_a$ answers, we can construct a dataset (denoted as $\mathcal{D}_{\texttt{IC-QA}}^{(k)}$) with $A_{n_q}^{k} \cdot A_{n_a}^{k} \cdot k$ examples, where $A_{n}^{m} = \frac{n!}{(n-m)!}$ denotes the number of ways to choose $m$ elements from a set of $n$ elements. Then, the task is to characterize that given any context block $\mathbf{E}_{\texttt{IC-QA}}^{(\pi,\pi')}(c) =[\mathbf{B}_{\texttt{IC-QA}}^{(\pi,\pi')}; \mathbf{q}_{\pi_c}; \mathbf{r}] $, whether the transformer can correctly output the desired answer $\mathbf a_{\pi_c'}$.

\subsection{Generalization: Template Matching Task} Motivated by the learning process of humans, where we abstract new things into different patterns for further analysis. Inspired by the template task designed in \citet{boix-adseraWhenCanTransformers2023a} for studying the generalization ability of transformer, we consider a similar template matching task to investigate whether transformers with varying attention layers have the ability to generalize. In particular, we first deliver the formal definition of the template.
\begin{definition}
\label{def:template}
A \textbf{template} is a string $\boldsymbol{t} \in {\mathcal{W}}^l$, where $\mathcal{W}$ is an alphabet of ``wildcards''. A \textbf{substitution map} is an injection function $s : \mathcal{W} \rightarrow \mathcal{X}$ that maps wildcards to real word symbols $\mathcal{X}$. Here, $\mathcal{X}$ can be seen as the alphabet of tokens in language or pixel blocks in an image. Different wildcards should be mapped to different tokens to ensure that each sequence can be mapped to one and only one template. We write $\texttt{sub}(\boldsymbol{t},s) \in {\mathcal{X}}^l$ for the sequence where each wildcard is substituted with the corresponding token: $\texttt{sub}(\boldsymbol{t}, s)_i = s(t_i)$. 
A \textbf{template labeling mapping} is a mapping from a template to the class index $f : {\mathcal{W}}^n \rightarrow \mathbb{Z}^*$.
\end{definition}

In general, the template can be understood as the abstract concept of the data, i.e., in \ref{tasks}, ``car'' is the concept of the car image with different colors.
To construct our dataset, we first define a template set with all possible templates of length $n_{\texttt{tmpl}}$: $\mathcal{T} = \{\boldsymbol{t}_0,\boldsymbol{t}_1,\dots,\boldsymbol{t}_{n_{\texttt{tmpl}}-1}\}$. We then use a \textit{template labeling mapping} $f$ to map each template $\boldsymbol{t}_i$ to a class $y_i$. After that, we use a set of $n_{\texttt{map}}$ \textit{substitution maps} $\mathcal{S} = \{s_0,s_1,\dots,s_{n_{\texttt{map}}-1}\}$ to generate data from the template to a real word sequence. We write the dataset as $\mathcal{D}_{\texttt{tm}} = \big\{\big(\texttt{sub}(\boldsymbol{t}_i,s_j),y_i)\big): i\in[n_{\texttt{tmpl}}], j\in[n_{\texttt{map}}]\big\}$, where $y_i = f(\boldsymbol{t}_i)$ denotes the template label and $\texttt{sub}(\boldsymbol {t}, s)$ denotes the sequence of real-word symbols that follow the template $\boldsymbol {t}$ and token mapping function $s$.

Similar to the sequence classification task, we also append a $\texttt{CLS}$ token at the end of the input sequence for generating the prediction. However, to investigate the generalization ability, the transformer cannot simply memorize all possible sequences but requires to learn their abstract patterns, i.e., the templates, to make the correct prediction. Then, the task is to characterize that given a sequence generated via $\texttt{sub}(\boldsymbol {t}_{k}, s_i)$, whether the transformer can identify the template $\boldsymbol {t}_k$ and output the correct prediction $y_k$. We call the model can \textit{generalize on template $\boldsymbol {t}_k$} if it can correctly predict all possible sequences generated by $\boldsymbol {t}_k$ and $s_i$.

\subsection{Contexture Generalization: In-context Template Matching Task } 

We then consider a more complex and general problem that is designed as the combination of in-context question-answering and template matching tasks, which requires the model to perform both reasoning and generalization simultaneously. This task is summarized in \ref{tasks}.

In particular, we formulate our problem by replacing the question in the context block $\mathbf{B}_{\texttt{IC-QA}}$ from a simple token $\mathbf{q}$ to the template data $\texttt{sub}(\boldsymbol{t},s)$. To construct our dataset, we need to define a set of templates $\mathcal{T} = \{\boldsymbol{t}_0,\boldsymbol{t}_1,\dots,\boldsymbol{t}_{l_{\texttt{tmpl}}-1}\}$. All templates have the same length $l$. Rather than pre-defining a mapping from the template $\boldsymbol{t}$ to a class label $y$, we follow the construction process in the previous in-context question-answering task. We first randomly choose $k$ templates from $\mathcal{T}$ and $k$ answers from the answer token set $\mathcal{A}$: $(\boldsymbol{t}_{\pi_0}, \dots, \boldsymbol{t}_{\pi_{k-1}})$ and $(\mathbf{a}_{\pi'_0}, \dots, \mathbf{a}_{\pi'_{k-1}})$. Then we can consider $k$ different substitution mapping function $s_{\pi''_0},\dots,s_{\pi''_{k-1}}$ for  $\boldsymbol{t}_{\pi_0}, \dots, \boldsymbol{t}_{\pi_{k-1}}$ to generate sequences of real-world symbols, denoted as $\mathbf X_{\pi_0},\dots,\mathbf X_{\pi_{k-1}}$, where $\mathbf X_{\pi_i} = \texttt{sub}(\boldsymbol {t}_{\pi_i}, s_{\pi''_{i}})$.
Then, the context block is defined as $\mathbf{B}_{\texttt{IC-TM}}^{(\pi,\pi')} = \big(\mathbf X_0,\mathbf{r} ,\mathbf{a}_{\pi'_0},\dots, \mathbf{X}_{\pi_{k-1}},\mathbf{r},\mathbf{a}_{\pi'_{k-1}})$. Then we randomly choose a query template $\boldsymbol {t}_{\pi_c}$ with $c\in\{0,\dots,k-1\}$ and use a new mapping function $s_{\pi''_k}$ to get the sequence of real-world symbols $\mathbf X_{\pi_c}$. Then, the entire input sequence is defined as $\mathbf{E}_{\texttt{IC-TM}}^{(\pi,\pi')}(c) = [\mathbf{B}_{\texttt{IC-TM}}^{(\pi,\pi')}, \mathbf X_{\pi_c}, \mathbf r]$, and the desired answer should be $\mathbf a_{\pi_c'}$. Then, the entire dataset, denoted as $\mathcal D_{\texttt{IC-TM}}$, is the collection of all sequences-answer pairs that generated by using all possible templates, answers, and mapping functions. 

Compared with the in-context learning question-answering and template matching tasks, this task requires the model to reason from the context and generalize to the unseen data. For instance, in \ref{tasks}, the model needs to first identify the template/concept of the query image (which is ``hat''), and then seeks the answer from the context (there is an example image using the same template and providing the answer ``hat''). In this task, the model should capture the similarity between each question (generalization) and retrieve the answer from the context (reasoning). 

\textbf{Summary and Discussion.}
We provide data examples and a more detailed comparison for the four tasks in \cref{sec:examples_data}. Note that we employ these tasks to assess the model's capacity, i.e. for the given architecture, especially the transformer with different attention layers, \textit{what} the model can do  and \textit{how} the model do it.  Specifically, we aim to determine whether there exists a particular configuration of the transformer model, such that all examples in the dataset can be perfectly learned. This ability is independent of the training process; our focus is solely on the ability of the transformer's architecture for tackling these tasks.

\section{Main Results}
\label{sec:main_results}
In this section, we present our main findings regarding the aforementioned tasks. We will focus on characterizing how transformer model performs on these tasks with varying attention layers. We will prove both negative and positive results on the capability of transformer when different numbers of attention layers are stacked.

\subsection{Single-Layer Transformer Can Memorize}
\label{sec:mem_main}
We commence our investigation by examining the memorization capability of a single-layer transformer. In this scenario, the model's objective is to accurately classify $N$ sequences with distinct labels. In particular, we will show that given sufficient heads, a single-layer transformer has the capability to memorize all data points. We summarize this result in the following Theorem.

\begin{theorem}
\label{prop:mem_succ}
   For any dataset of the sequence classification task, denoted by $D_{\texttt{SC}}$, let $d$ be the token dimension, and $n$ be the length of the sequence (i.e., number of tokens). Then there exists a transformer $\texttt{TF}$ with $L = 1$ attention layer, $n$ attention heads, and model embedding dimension $d' = \max\{nd,d+n\}$ such that for all $(\mathbf X, y)\in\mathcal D_{\texttt{SC}}$, it holds that $\texttt{TF}(\mathbf X)=y$\footnote{Here we slightly abuse the notation use $\texttt{TF}(\mathbf X)$ the denote the prediction result of the input $\mathbf X$. Similar notations will be used in other theorems.}.
\end{theorem}
We first remark that the goal of \cref{prop:mem_succ} is to demonstrate the ability of the single-layer transformer for the memorization task, while the (horizontal) model size, i.e., number of heads and embedding dimensions,  are not optimized. It is possible to further sharpen our analysis, e.g., applying the techniques in \citet{mahdaviMemorizationCapacityMultiHead2023}, to relax the conditions on the (horizontal) size of the transformer model.

To achieve this, we show that a single attention layer, with $n$ attention heads, can perform the \textbf{mapping} operation to transformer the input sequence, formulated as a matrix of embeddings (of dimension $\mathbb R^{d\times n}$), to a distinct vector representation. Moreover, we show that these vector representations are linearly independent. 
Then the output classifier layer, equipped with the weight matrix  $\mathbf{W}_{O} \in \mathbb{R}^{N \times d'}$ ($N$ denotes the number of total labels), can map each vector representation to a probability vector, where the index of the largest entry corresponds to the desired sequence label. The full proof and construction of the transformer weights can be found in \cref{sec:prof_mem}.

\cref{prop:mem_succ} demonstrates that the one attention layer is sufficient for memorization. However, it is important to note that memorization alone cannot guarantee other more challenging and critical abilities such as reasoning and generalization. Characterizing the ability of transformer in these aspects will be the focus of the subsequent subsections. 

\subsection{Two-Layer Transformer Performs Reasoning}
\label{sec:icl_main}
Then we explored the ability of transformers to reason using simple in-context learning tasks. Previous research has investigated similar tasks, using induction heads \cite{olsson2022context} and transformer circuits \cite{elhage2021mathematical}, to assess the transformer's reasoning ability. However, the theoretical basis for these observations remains unclear, the connection between the number of attention layers and the reasoning ability has not been thoroughly studied.

In this section, we theoretically characterize the reasoning performance of single-layer and two-layer transformer models on the in-context question-answering task. First, we provide the following theorem to show that any single-layer transformer cannot perfectly perform the reasoning task. 
\begin{theorem}
\label{prop:lsa_fail}
Let $\mathcal D_{\texttt{IC-QA}}$ be a dataset of the in-context question-answering task and  $n$ be the number of question-answer pairs. Then for any transformer with $L = 1$ attention layer, no matter how many heads are applied, there exists at least one data point $(\mathbf{E}_{\texttt{IC-QA}}^{(\pi,\pi')}(c), \mathbf a_{\pi_c'})\in \mathcal D_{\texttt{IC-QA}}^{(k)}$ such that $\texttt{TF}\big(\mathbf{E}_{\texttt{IC-QA}}^{(\pi,\pi')}(c)\big)\neq \mathbf a_{\pi_c'}$.
\end{theorem}


\cref{prop:lsa_fail} suggests for any single-layer transformer, there must exist at least one data point that cannot be correctly predicted, suggesting its inability to perfectly tackle the in-context question-answer task. The idea to prove this is to show that single-layer attention function can preserve the linear dependency (defined in terms of the set operations, see \cref{sec:limitation_dependency}). In other words, if multiple input sequences, such as the entire dataset, exhibit some dependence, then the corresponding outputs of the single-layer attention will also display linear dependence. By leveraging this linear dependence in the outputs, we can demonstrate that the attention function fails to successfully learn all question-answering tasks. The detailed proof can be found in \cref{sec:prf_lsa_fail}.

Moreover, we claim that \cref{prop:lsa_fail} is not limited to the ReLU attention, but can also apply to softmax attention when using single head in\cref{sec:soft_max} (extending to multiple head case is left for future study). Then, we show that, in the following theorem, a two-layer transformer can resolve the issue of the first-layer transformer and perfectly reason all sequences in the dataset. 


\begin{theorem}
\label{prop:lsa_succ}
For any dataset of the in-context question-answering task, denoted by $\mathcal D_{\texttt{IC-QA}}$, let $k$ be the number of question-answer pairs (the sequence length is $n=3k+2$) and $d$ be the dimension of the token embedding. There exists a transformer $\texttt{TF}$ with $L = 2$ attention layers, $1$ attention head, and $d' = d+n$ such that for all $\big(\mathbf{E}_{\texttt{IC-QA}}^{(\pi,\pi')}(c),\mathbf a_{k_c'}\big)\in \mathcal D_{\texttt{IC-QA}}$, it holds that $\texttt{TF}(\mathbf{E}_{\texttt{IC-QA}}^{(\pi,\pi')}(c))= \mathbf a_{k_c'}$.
\end{theorem}

Our proof, i.e., the construction of such a two-layer transformer model, draws inspiration from  \cite{friedman2023learning}, which shows that the two-layer transformer can perform a \textbf{copying-matching} procedure to accomplish the template matching task.  We construct the first layer to perform the \textbf{copying} operation among question and answer tokens to aggregate each question and the corresponding answer together. The second layer is implemented as an induction head \cite{olsson2022context} to perform the \textbf{matching} operation between the token representations (which already aggregate the question and answer together) with the same question (i.e., the query question), and then output the desired answer. The detailed construction is in \cref{sec:cons_icl}.

By combining \cref{prop:lsa_fail} and \cref{prop:lsa_succ}, we can conclude that it requires two attention layers to perfectly perform the reasoning. However, the 2-layer attention-only transformer can do more than just copying and matching. Next, we will show that 2-layer transformers can also accomplish the generalization task through a different mechanism.

\subsection{Two-Layer Transformer Can Generalize}
\label{sec:tmpl_main}

In this part, we shift our focus to the generalization ability of transformers. Specifically, we consider the template matching task, where each template has a distinct label, and sequences that follow from the same template will be assigned by the same label. Our goal is to investigate whether and how transformers can successfully perform this task, i.e., identify the template of the input sequence and predict its label, for all possible sequences. This serves as the necessary condition for the generalization of transformer \citep{boix-adseraWhenCanTransformers2023a}. Similar to the findings in \cref{sec:icl_main}, we also observe that a single-layer transformer fails to accurately learn this generalization process, which is summarized in the following theorem.


\begin{theorem}
\label{prop:tmplt_fail} Let $\mathcal D_{\texttt{TM}}$ be a dataset of the template matching task and $n$ be the sequence length. Then for any transformer with $L = 1$ attention layer, no matter how many heads are applied, there exists at least one data $(\texttt{sub}(\boldsymbol {t},s), y)\in \mathcal D_{\texttt{TM}}$, generated via a template $\boldsymbol {t}$ and a mapping $s$, such that $\texttt{TF}(\texttt{sub}(\boldsymbol {t},s))\neq y$.
\end{theorem}

We follow a similar idea for proving \cref{prop:lsa_fail} to prove the above argument. In particular, we can show that there are two templates such that all possible sequences generated accordingly are linearly dependent. Then if these two templates have different labels, the single-layer transformer fails to correctly classify all sequences generated via these two templates. The detailed proof can be found in \cref{sec:prf_tmplt_fail}. 

This intriguing result suggests that although single-layer transformers possess strong memorization abilities, they struggle with more complex tasks. Moreover, we show that this template matching task can be performed by a two-layer transformer, which is stated in the following theorem.

\begin{theorem}
\label{prop:tmplt_succ}
For any dataset of the template matching task, denoted by $\mathcal D_{TM}$, let $n$ be the sequence/template length and $d$ be the token embedding dimension. Then there exists a transformer $\texttt{TF}$ with $L = 2$ attention layers, $1$ attention heads, and $d' = d+n$  such that for all $(\texttt{sub}(\boldsymbol {t},s), y)\in \mathcal D_{\texttt{TM}}$, it holds that $\texttt{TF}(\texttt{sub}(\boldsymbol {t},s)) = y$.
\end{theorem}

We show that such a two-layer transformer can be constructed using a \textbf{parsing-mapping} process. In particular, the first layer can be designed to parse the sequence into the corresponding template, then the second layer can perform a memorization process that is similar to the sequence classification task investigated in Section \cref{sec:mem_main}. These findings prompt us to reconsider the mechanism of multi-layer transformers, instead of solely relying on memorizing all the data \citep{yunAreTransformersUniversal2020}. The detailed construction can be found in \cref{sec:cons_tmplt}.

\subsection{Three-Layer Transformer Can Perform Contextual Generalization}
\label{sec:cont_gen}

In previous sections, we have shown that a 2-layer transformer is capable of conducting reasoning and generalization tasks. Now we will focus on a more challenging in-context template matching task that requires the model to perform generalization and reasoning simultaneously, i.e., exhibiting the contextual generalization capability. 

First, since the in-context template matching task can degenerate to the standard in-context question-answering task (e.g., using identity mapping from the template alphabet to real-world symbols). Then, we can straightforwardly leverage the result in \cref{prop:lsa_fail} to demonstrate the failure of the single-layer transformer in accomplishing this task. Moreover, note that when tackling the in-context question-answering and template matching tasks, the transformer is constructed to perform two-step copy-matching and parse-mapping procedures, respectively. Therefore, regarding the in-context template matching task, we can design a transformer to perform a  three-step \textbf{parsing-copying-matching} procedure, which is constructed using three attention layers. We state this result in the following theorem.

\begin{theorem}\label{prop:context_gen_succ}
For any dataset of the in-context template matching task, denoted by $\mathcal D_{\texttt{IC-TM}}$, let $l$ be the template length, $k$ be the number of question-answer pairs (then the sequence length is $n=k(l+2)+l+1$), and $d$ be the dimension of the token embedding. There exists a transformer $\texttt{TF}$ with $L = 3$ attention layers, $2l$ attention heads, and $d' = d+n+l+2$ such that for all $\big(\mathbf{E}_{\texttt{IC-TM}}^{(\pi,\pi')}(c),\mathbf a_{k_c'}\big)\in \mathcal D_{\texttt{IC-TM}}$, it holds that $\texttt{TF}(\mathbf{E}_{\texttt{IC-TM}}^{(\pi,\pi')}(c))= \mathbf a_{k_c'}$.
\end{theorem}

We remark that \cref{prop:context_gen_succ} does not imply that the in-context template matching task cannot be accomplished by two-layer transformers. However, in our numerical experiments (see \cref{fig:loss_acc_1}), we find that the two-layer transformer struggles with this task and can even not perform well during the training. Therefore, we tend to believe that a three-layer transformer may be the shallowest one to perform contextual generalization, while the rigorous proof for the failure of two-layer models is left for future study.

\section{Experiments}
\label{sec:exp}
\begin{figure*}
\vskip -.1in
\centering
\begin{subfigure}{0.23\textwidth}
  \centering
  \includegraphics[width=0.845\textwidth]{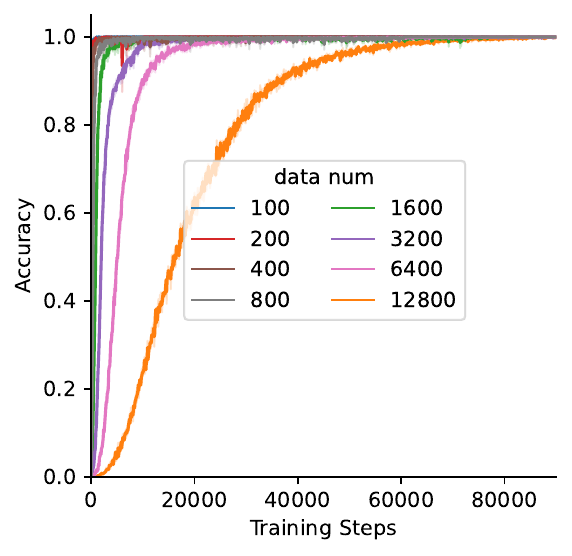}
\end{subfigure}
\begin{subfigure}{0.23\textwidth}
  \centering
  \includegraphics[width=0.9\textwidth]{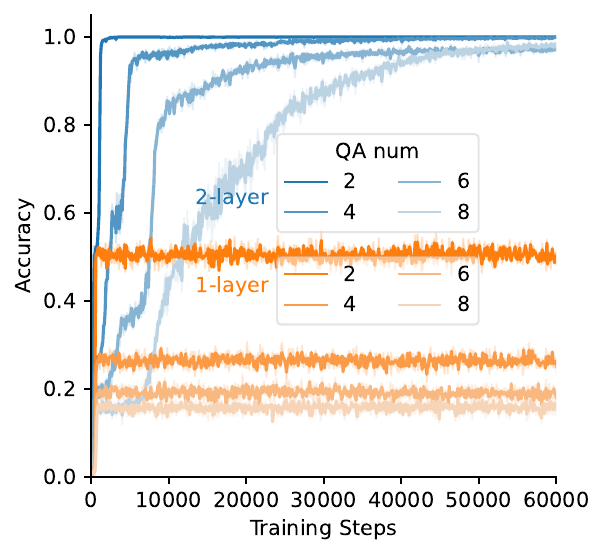}
\end{subfigure}
\begin{subfigure}{0.23\textwidth}
  \centering
  \includegraphics[width=0.845\textwidth]{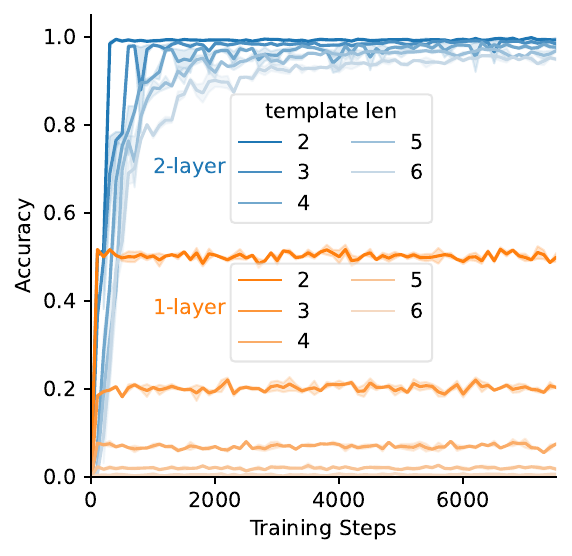}
\end{subfigure}
\begin{subfigure}{0.23\textwidth}
  \centering
  \includegraphics[width=0.9\textwidth]{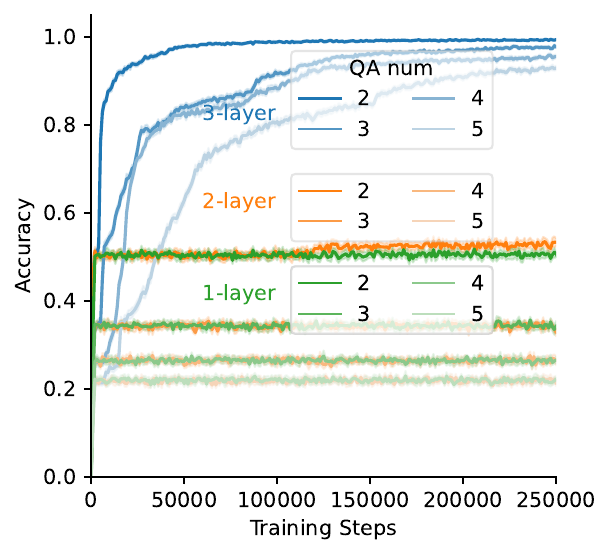}
\end{subfigure}
\begin{subfigure}{0.23\textwidth}
  \centering
  \includegraphics[width=0.845\textwidth]{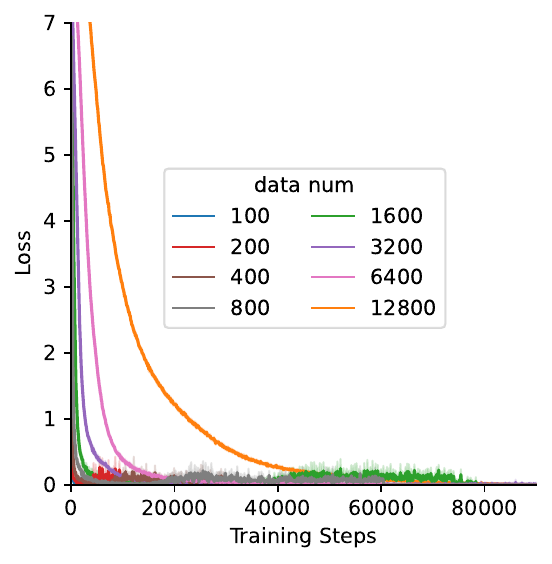}
\end{subfigure}
\begin{subfigure}{0.23\textwidth}
  \centering
  \includegraphics[width=0.9\textwidth]{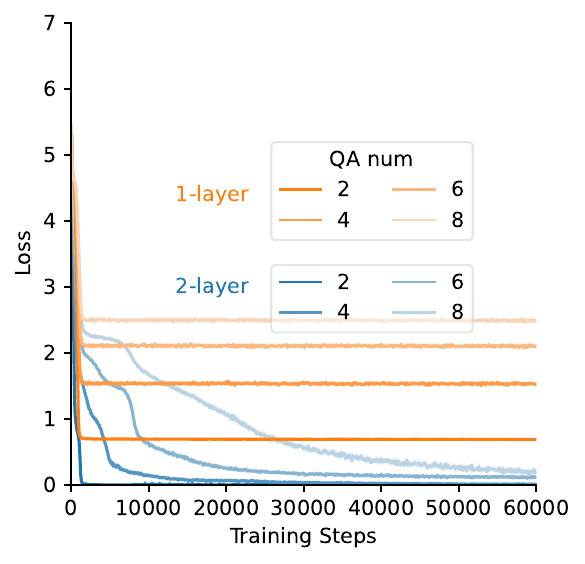}
\end{subfigure}
\begin{subfigure}{0.23\textwidth}
  \centering
  \includegraphics[width=0.845\textwidth]{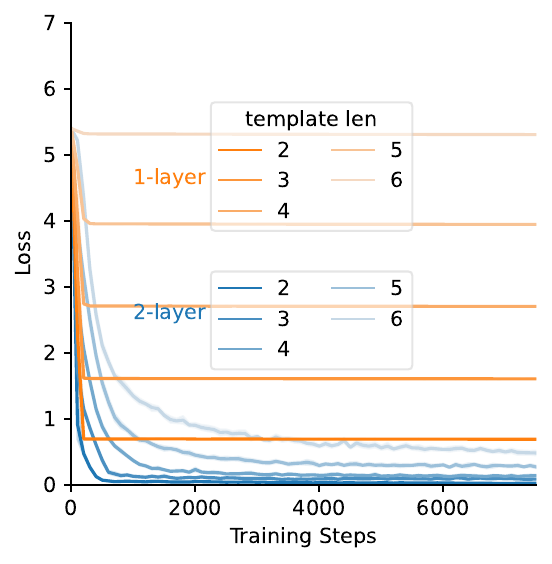}
\end{subfigure}
\begin{subfigure}{0.23\textwidth}
  \centering
  \includegraphics[width=0.9\textwidth]{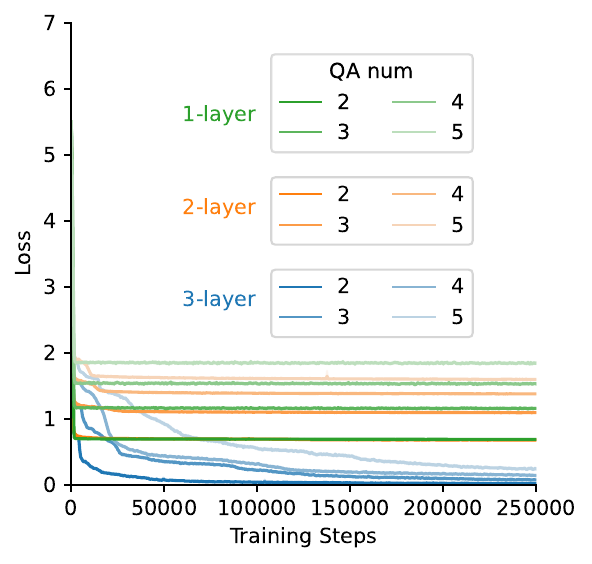}
\end{subfigure}
\vskip -.1in

\caption{Performance of different layers of transformers on memorization, reasoning, generalization, and contextual generalization tasks. \textit{Far left column}: A single-layer transformer can memorize sequences with distinct labels. \textit{Center left column}: A single-layer transformer struggles with reasoning tasks, while a two-layer transformer can learn reasoning with enough training steps. \textit{Center right column}: A single-layer transformer struggles with generalizing on template tasks, while a two-layer transformer can quickly grasp the method for generalization. \textit{Far right column}: When it comes to more complex contextual generalization tasks, a 1/2-layer transformer fails, but a 3-layer transformer can perform well on such tasks.}
\label{fig:loss_acc_1}
\vskip -.1in

\end{figure*}
In this section, we verify the main results presented in \cref{sec:main_results} through synthetic datasets. We examine the accuracy and loss dynamics for the four tasks across different layers and heads of transformers. Additionally, we study the reasoning and generalization mechanisms of transformers by analyzing attention maps and comparing them with our constructed transformer. The detailed experimental setup is presented in \cref{sec:s_exp}.

\subsection{The Impact of Attention Layers on Different Tasks}

We begin by studying the impact of transformer depth on these tasks. The results are shown in \cref{fig:loss_acc_1}. We observe that a single-layer transformer performs well on the memorization task but struggles with tasks related to generalization and reasoning. This validates  Theorems \cref{prop:mem_succ}, \cref{prop:lsa_fail},  \cref{prop:lsa_succ}, \cref{prop:tmplt_fail}, and \cref{prop:tmplt_succ}.
Single-layer transformer performs like a random guess for generalization and reasoning tasks.  

For the contextual generalization task, we interestingly find the same random guessing degeneration in both single-layer and two-layer transformers, indicating that a 2-layer transformer might not be able to handle such a complex task that requires both generalization and reasoning. Instead, a $3$ layer transformer performs perfectly on this task. This validates  \cref{prop:context_gen_succ}. Through this task, we can observe the emergence of more complex reasoning and generalization when we extend the layer of the transformer from 2 to 3. In \cref{sec:s_exp}, we also study the performance of a 4-layer transformer, which shows that compared to a 3-layer transformer, a 4-layer transformer can perform contextual generalization more quickly. This emphasizes the effectiveness of using deeper model to perform harder tasks.

\subsection{Algorithms Behind Trained Transformers}

To further understand how transformers achieve generalization and reasoning, we analyze the attention maps for some typical examples. The results show that trained transformers exhibit similar mechanisms to our constructions.



In the reasoning task, we observe operations in \cref{fig:attn_icl} that are similar to the constructed copying-matching mechanism. In particular, we can identify two ``copying'' operations (the corresponding value in the attention map is relatively high) in the first layer: in \cref{fig:attn_icl} (\textit{top row}), tokens \texttt{"1"} and \texttt{"A"} are copied to the $4$-th and last positions respectively. In the second layer, we can then identify a ``matching'' operation: token \texttt{"1"}, which now appears in the $3$-th position, strongly correlates with the token \texttt{"A"}, which now appears in the last position. This further leads to the correct answer \texttt{"1"}. Similar observation can be found in the second example: token \texttt{"2"} and \texttt{"B"} are copied to the first and last positions respectively in the first layer; then a matching between the token \texttt{"B"} in the last position and the value \texttt{"2"} in the first position occurs, leading to the correct answer.

In the template matching task, we also find evidence of our constructed parsing-mapping mechanism in \cref{fig:attn_tmplt}. Specifically, a "parsing" operation that checks the similarity of tokens in other positions can be observed in the first layer: for input sequence \texttt{"AAB="}, the repeat token \texttt{"A"} in positions $0$ and $1$ share attention with each other, for the input \texttt{"ABB="}, the repeat token in positions $1$ and $2$ share the position information. In this way, the model parses the input sequence into a template representation, $\texttt{AAB} \rightarrow \alpha \alpha \beta$ and $\texttt{ABB} \rightarrow \alpha \beta \beta$, which can be mapped to different templates by utilizing the memorization ability of the transformer.




\begin{figure}[t]
\begin{minipage}[t]{0.47\textwidth}
\centering
\includegraphics[width=\columnwidth]{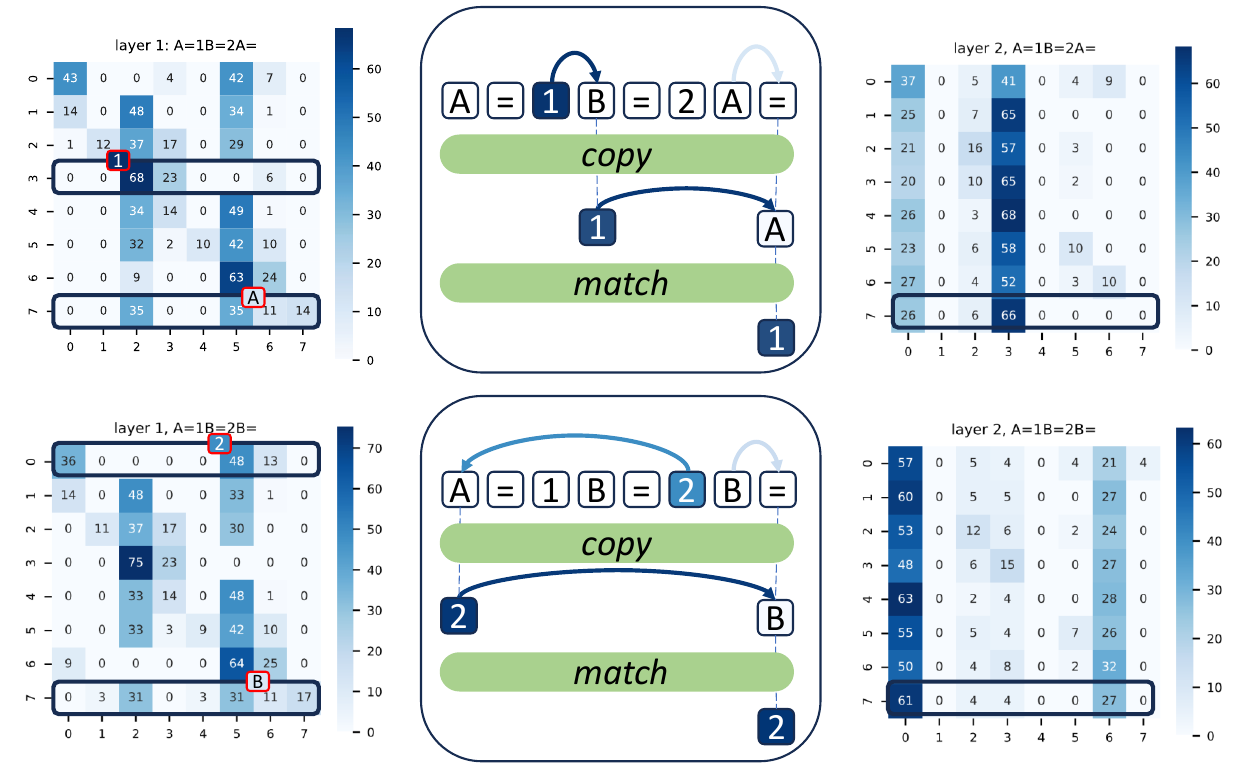}
\caption{Attention maps for a trained two-layer transformer in the reasoning sequences \texttt{"A=1B=2A="} (\textit{top row}) and \texttt{"A=1B=2A="} (\textit{bottom row}). }
\label{fig:attn_icl}
\end{minipage}
\hfill
\begin{minipage}[t]{0.47\textwidth}
\centering
\includegraphics[width=\columnwidth]{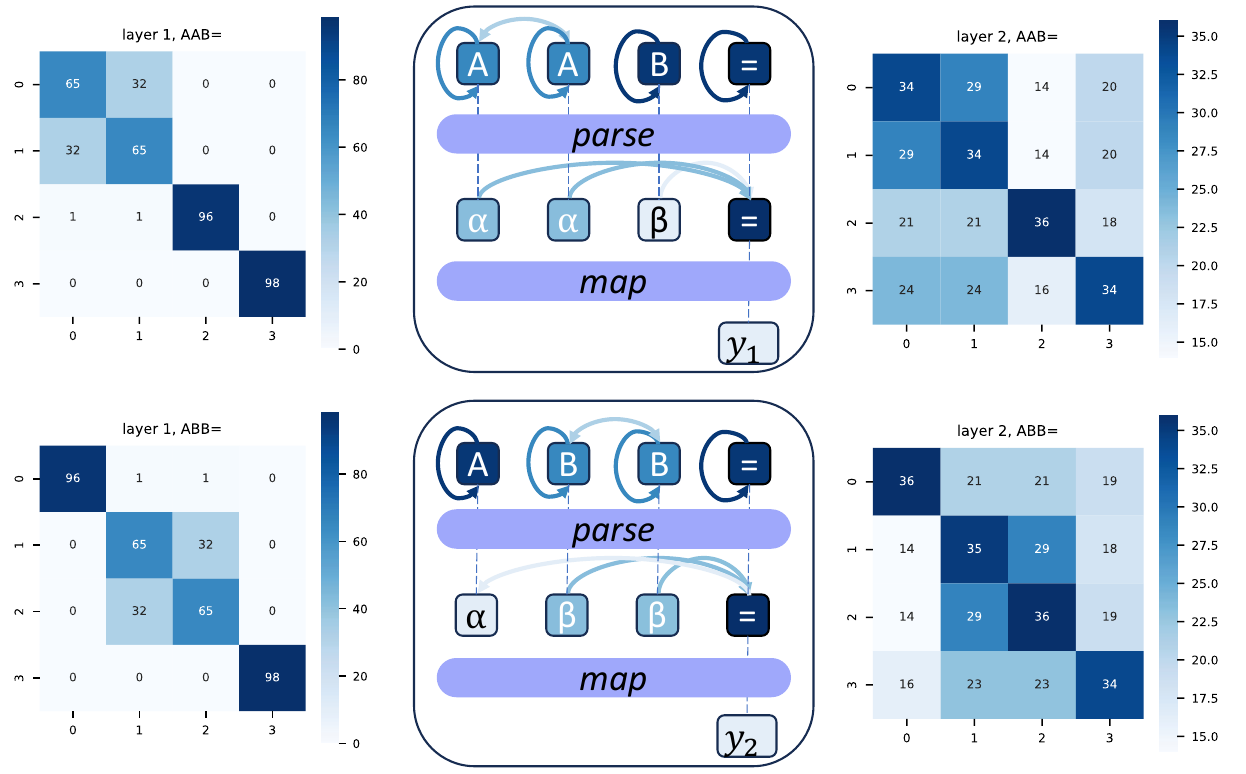}
\caption{Attention maps for a trained two-layer transformer in the template sequences \texttt{"AAB="} (\textit{top row}) and \texttt{"ABB="} (\textit{bottom row}).}
\label{fig:attn_tmplt}
\end{minipage}
\end{figure}

\section{Discussion}
In this study, we explore the capabilities of transformers with varying attention layers in performing various tasks, including memorization, reasoning, generalization, and contextual generalization. Our investigation reveals the limitations of single-layer transformers when dealing with complex tasks, and highlights the importance of using multiple attention layers to achieve optimal performance in reasoning, generalization, and contextual generalization tasks.  Our findings shed light on the theoretical properties of transformer models, offering insights into their design and optimization for diverse tasks. Besides, our framework can be further expanded to more challenging tasks. For example, we can expand our in-context QA as "nested in-context QA" task, where the model must perform a chain-of-thought process to arrive at the final answer. This could involve a sequence like ``$\texttt{a\,\textrightarrow\,b\,b\,\textrightarrow\,c\,c\,\textrightarrow\,d\,a\textrightarrow\,{\underline{d}}}$ '',  where we can design a  transformer with 6 layers that performs "copy-matching" 3 times to solve this problem effectively. We believe that expanding our four tasks is worth for further investigation, and our framework can offer valuable insights.
Moreover, our analysis, which determines the threshold for attention layers in solving complex tasks, can even offer explanations for the scalability of transformers \cite{kaplanScalingLawsNeural2020a} and their emergent abilities \cite{weiEmergentAbilitiesLarge2022}.

\bibliography{iclr2024_conference}
\bibliographystyle{plainnat}

\appendix
\section*{Organization of the Appendix}

\begin{itemize}
    \item Examples for Different Tasks
    \item Supporting Experiments
    \item Useful Transformer Constructions
    \item Limitation for Single Layer Attention-only Transformer
    \item Proofs for \cref{sec:mem_main}
    \item Proofs for \cref{sec:icl_main}
    \item Proofs for \cref{sec:tmpl_main}
    \item Construction for \cref{sec:cont_gen}
    \item Expanding from ReLU to Softmax Attention
\end{itemize}

\section{Examples for Different Tasks}
\label{sec:examples_data}

Here we use letters $\texttt{a},\texttt{b},\dots$ to denote the ``question tokens'' (also can be seen as pixel blocks for input images in \cref{tasks}, use integer number $\texttt{0},\texttt{1},\texttt{2},\dots$ denote the labels, such as ``hat'' and ``car'', and use $\texttt{\textrightarrow}$ denote the response sign (or \texttt{CLS}), then we provide data examples as in \cref{tab:4-tasks},  for each example, the prior part of the sequence is the input, and the model should predict the underlined result:

\begin{table}[H]

\caption{Examples of four tasks.}
\begin{center}
\begin{small}
\begin{tabular}{ccc}
\toprule
Task & Data Example & Explanation\\
\midrule
Memorization &  $\substack{\texttt{aa\,\textrightarrow\,{\underline{1}} \quad {bb}\,\textrightarrow\,{\underline{2}} \quad {cc}\,\textrightarrow\,{\underline{3}}} 
\\ \texttt{ab\,\textrightarrow\,{\underline{4}} \quad \texttt{ac}\,\textrightarrow\,{\underline{5}} \quad \texttt{ba}\,\textrightarrow\,{\underline{6}}}
\\ \texttt{bc\,\textrightarrow\,{\underline{7}} \quad \texttt{ca}\,\textrightarrow\,{\underline{8}} \quad \texttt{cb}\,\textrightarrow\,{\underline{9}}}}$ & $\substack{\text{different sequence belong to different class}}$ \\
\midrule
Reasoning &  $\substack{\texttt{a\,\textrightarrow\,1\,b\,\textrightarrow\,2\,a\,\textrightarrow\,{\underline{1}} \quad  b\,\textrightarrow\,1\,a\,\textrightarrow\,2\,a\,\textrightarrow\,{\underline{2}}} \\ \texttt{a\,\textrightarrow\,1\,b\,\textrightarrow\,2\,b\,\textrightarrow\,{\underline{2}} \quad b\,\textrightarrow\,1\,a\,\textrightarrow\,2\,b\,\textrightarrow\,{\underline{1}}}\\}$ & $\substack{\text{answer the last question based on the context}}$ \\
\midrule
Generalization & $\substack{\texttt{aa\,\textrightarrow\,{\underline{1}} \quad {bb}\,\textrightarrow\,{\underline{1}} \quad {cc}\,\textrightarrow\,{\underline{1}}} 
\\ \texttt{ab\,\textrightarrow\,{\underline{2}} \quad \texttt{ac}\,\textrightarrow\,{\underline{2}} \quad \texttt{ba}\,\textrightarrow\,{\underline{2}}}
\\ \texttt{bc\,\textrightarrow\,{\underline{2}} \quad \texttt{ca}\,\textrightarrow\,{\underline{2}} \quad \texttt{cb}\,\textrightarrow\,{\underline{2}}}}$ & $\substack{\text{sequence generated by same template belongs to same class}\\ \text{different template have different label} \\ (\text{in this example we have templates } \alpha \alpha \rightarrow 1, \, \alpha \beta \rightarrow 2)}$  \\
\midrule
\begin{tabular}[c]{@{}c@{}}Contextual\\ Generalization\end{tabular} & $\subalign{\texttt{aa\,\textrightarrow\,1\,ab\,\textrightarrow\,2\,bb\,\textrightarrow\,{\underline{1}} \quad aa\,\textrightarrow\,1\,ab\,\textrightarrow\,2\,ba\,\textrightarrow\,{\underline{2}}} \\ \texttt{aa\,\textrightarrow\,1\,ab\,\textrightarrow\,2\,aa\,\textrightarrow\,{\underline{1}} \quad aa\,\textrightarrow\,1\,ab\,\textrightarrow\,2\,ab\,\textrightarrow\,{\underline{2}}}}$ & $\substack{\text{similar question have the same answer} \\ \text{answer the last question based on the context}}$ \\
\bottomrule
\end{tabular}
\end{small}
\end{center}
\label{tab:4-tasks}
\end{table}

\paragraph{Memorization Task:} We assume that each sequence belongs to a different cluster, meaning that the sequences are "independent" from each other. This means that the model only needs to memorize each sequence and its corresponding label without considering any relationships between the data.

\paragraph{Reasoning Task:} We first provide question-answer pairs and then ask the model to retrieve the answer from the context. Such task can be used as a benchmark to evaluate the model's reasoning and comprehension ability \cite{liu2023lost}.

\paragraph{Generalization Task:} we assume that sequences generated by the same template belong to the same class. For example, different-colored cars should be classified into the same class. Our setting is adapted from \cite{boix-adseraWhenCanTransformers2023a}, and we focus on understanding when and how the model can generalize to \textbf{all possible} sequences that belong to the same template.

\paragraph{Contextual Generalization:} We assume that similar questions (generated by the same template) should have the same answer. Therefore, the task requires analyzing the semantic similarity between each question and retrieving the answer from the context.

Reasoning focuses on the semantic meaning behind the sequence, while generalization focuses on the relationships between each sequence. This makes our task more challenging compared to the memorization task, which ignores any possible relationships within and between the data.

\section{Supporting Experiments}
\label{sec:s_exp}
\subsection{Experiments setup}

\textbf{Model:} Our model is an attention-only transformer with a classification layer. We employ different initialization methods for different tasks. For the memorization task, we initialize $\mathbf{W}_{QK}$ with a uniform distribution in the range $[0,1)$. For other tasks, we initialize $\mathbf{W}_{QK}$ in the first layer of the transformer with constructed pattern for different tasks, as illustrated in \cref{sec:cons_icl}, \cref{sec:cons_tmplt}, and \cref{sec:cons_cont_gen}. This allows us to provide a favorable starting point for these models. Additionally, we initialize $\mathbf{W}_V$ for each task with a value of $1$. The classification layer $\mathbf{W}_{C}$ is initialized with a uniform distribution in the range $[0,1)$. We have observed that tasks related to generalization and initialization can be challenging even for randomly initialized transformers with sufficient layers. Our objective is to gain insights into the workings of transformers on these tasks and empirically validate the propositions in \cref{sec:main_results}. Hence, this choice of initialization is justified.

\textbf{Data:} We construct the dataset as described in \cref{sec:problem_set}. Before inputting the data to the transformer, we concatenate each token with a one-hot positional encoding. This is done to ensure that different types of tokens are disentangled, meaning that each token is encoded in a certain subspace. For memorization tasks, we set the word dimension $d_1$ to $500$ and add an extra dimension for the response sign, resulting in a final token dimension of $d = 500 + 1$. The sequence length is set to $6$, so the one-hot positional encoding has a dimension of $n = 6+1$. We randomly choose $e$ examples and assign them with random permuted $e$ distinct labels. Therefore, the input sequence is represented as a $(d + n) \times n$ matrix, where the transformer hidden size is set to $d' = (d + n) = 508$. 

Similarly, for reasoning tasks, we set the question dimension $d_1$ to $100$, the answer dimension $d_2$ to $100$, and include an extra dimension for the response sign. This results in a token dimension of $d = 100+100+1$. For a sequence with $k$ examples, the total length is $n = 3k + 2$.

For generalization tasks, we set the word dimension to $d_1 = 210$ and add an extra dimension for the response sign. The input sequence length is $n = l + 1$, where $l$ represents the template length.

For contextual generalization tasks, we set the question dimension $d_1 = 100$, the answer dimension $d_2 = 100$, and each question is generated from a template length $l = 5$. For a sequence with $k$ examples, the total length is $n = (l + 1 +1)(k+1) - 1 = 7k + 6$.

\textbf{Training:} During training, we utilize stochastic gradient descent (SGD) as the optimizer with cross-entropy loss function given by:

$$
\ell(y,\mathbf{o}) = -\log \frac{e^{\mathbf{o}[y]}}{\sum e^{\mathbf{o}[i]}}
$$

Here, $\hat{\mathbf{o}}$ represents the prediction result based on the last response sign, and $y$ represents the target index. In the case of memorization and template generalization tasks, $y$ can be viewed as a label. For reasoning and contextual reasoning tasks, $y$ corresponds to a vocabulary index.

\subsection{The Impact of the Number of Attention Heads}

In our previous experiments, we discovered that a single-head transformer with 2 layers is sufficient for performing reasoning and generalization tasks. However, in practice, it is common to use multiple attention heads, so we conducted additional experiments to investigate the effects of using more than one attention head. In the following figure, we represent the number of heads on each layer using the list $[h_1,h_2,\cdots]$.

\begin{figure}[H]
\begin{center}
\begin{subfigure}{0.4\columnwidth}
    \centering
    \includegraphics[width=0.7\columnwidth]{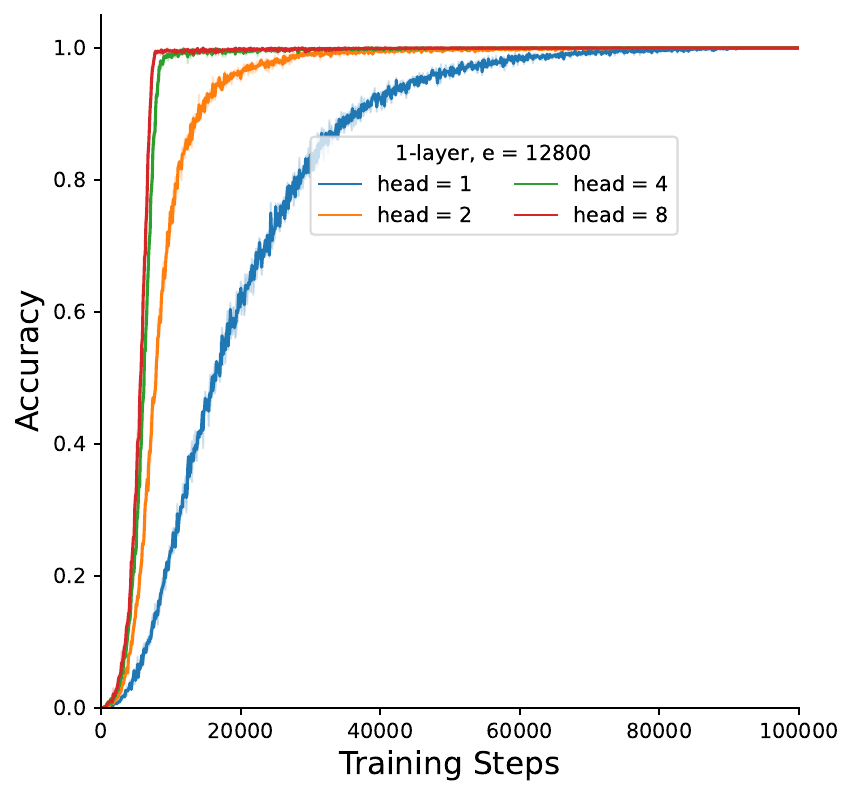}
\end{subfigure}
\begin{subfigure}{0.4\columnwidth}
    \centering
    \includegraphics[width=0.7\columnwidth]{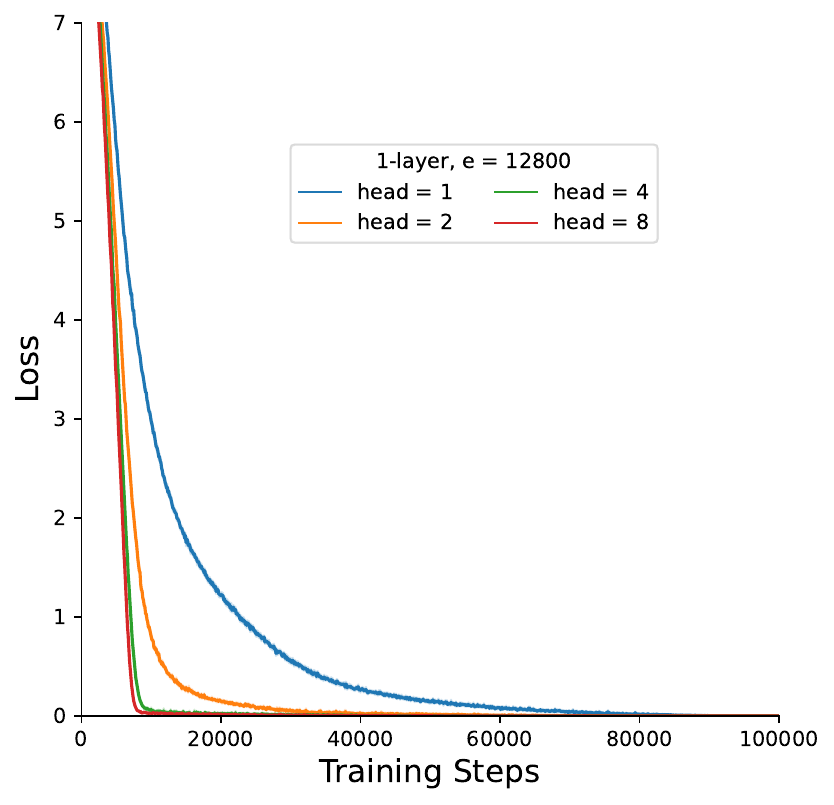}
\end{subfigure}
\caption{Training dynamic for different attention heads on memorization task}
\label{fig:add_mem_head}
\end{center}
\end{figure}
\begin{figure}[H]
\begin{center}
\begin{subfigure}{0.4\columnwidth}
    \centering
    \includegraphics[width=0.7\columnwidth]{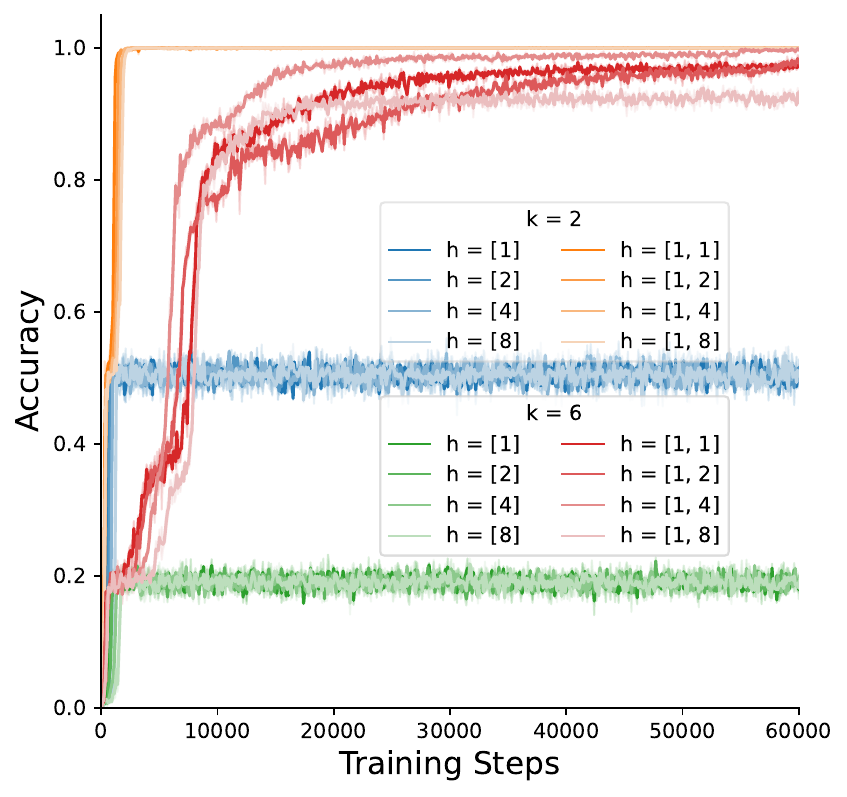}
\end{subfigure}
\begin{subfigure}{0.4\columnwidth}
    \centering
    \includegraphics[width=0.7\columnwidth]{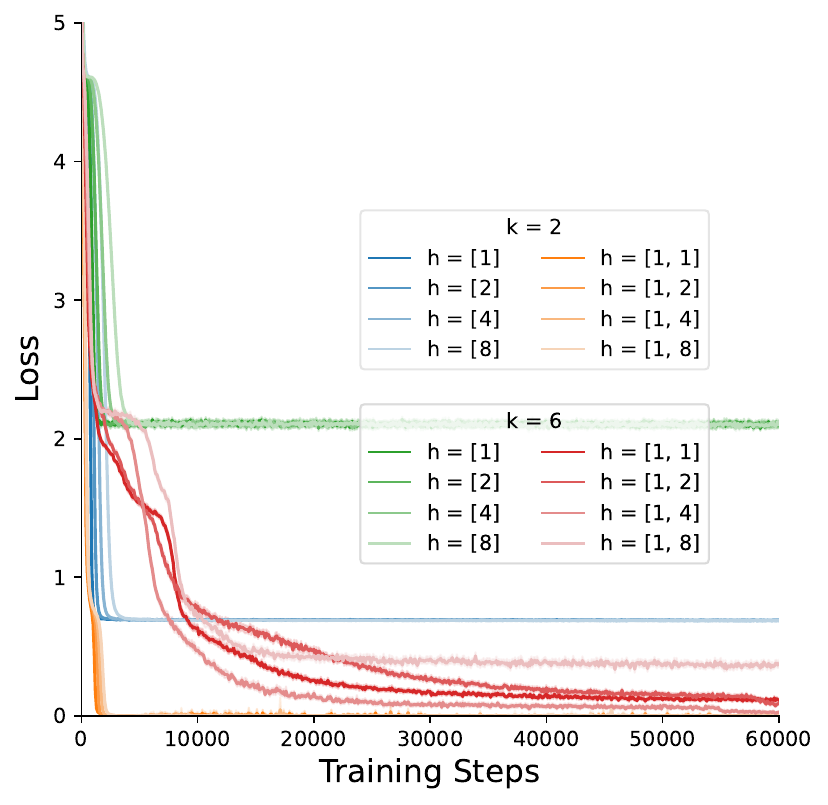}
\end{subfigure}
\caption{Training dynamic for different attention heads on in-context learning task}
\label{fig:add_icl_head}
\end{center}
\end{figure}
\begin{figure}[H]
\begin{center}
\begin{subfigure}{0.4\columnwidth}
    \centering
    \includegraphics[width=0.7\columnwidth]{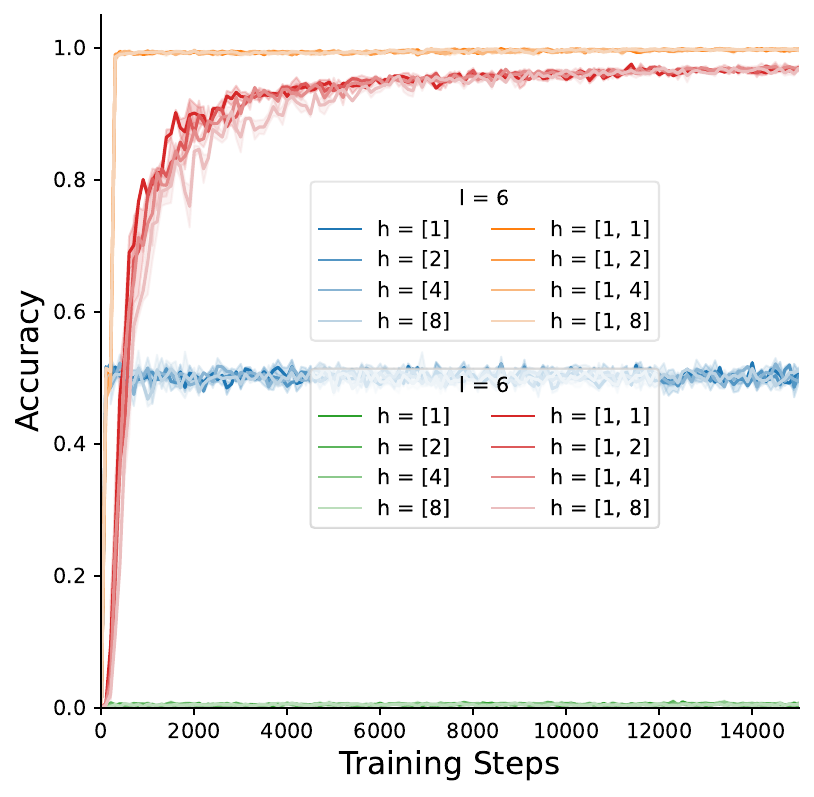}
\end{subfigure}
\begin{subfigure}{0.4\columnwidth}
    \centering
    \includegraphics[width=0.7\columnwidth]{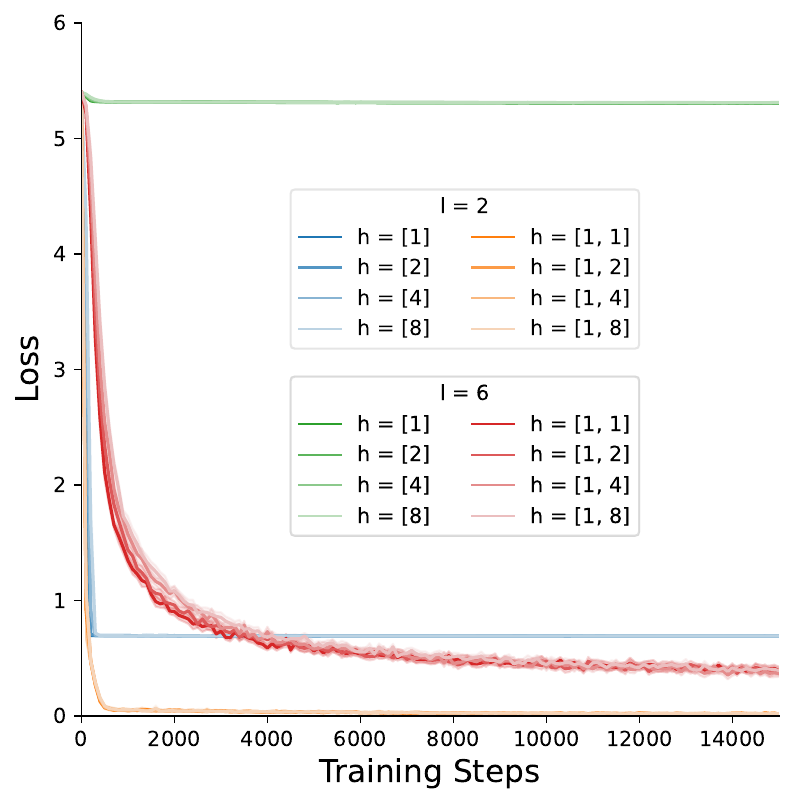}
\end{subfigure}
\caption{Training dynamic for different attention heads on template generalization task}
\label{fig:add_tmplt_head}
\end{center}
\end{figure}

For the memorization task, increasing the number of attention heads allows the model to more easily memorize the data, as shown in \cref{fig:add_mem_head}. However, for tasks that require reasoning and generalization, such as those depicted in \cref{fig:add_icl_head} and \cref{fig:add_tmplt_head},  additional attention heads does not yield a significant impact.

\subsection{Additional experiments for contextual generalization task}

When it comes to more challenging contextual generalization task, we conduct extensive additional experiments to investigate the impact of both the number of layer and attention heads.

\begin{figure}[H]
\begin{center}
\begin{subfigure}{0.4\columnwidth}
    \centering
    \includegraphics[width=0.7\columnwidth]{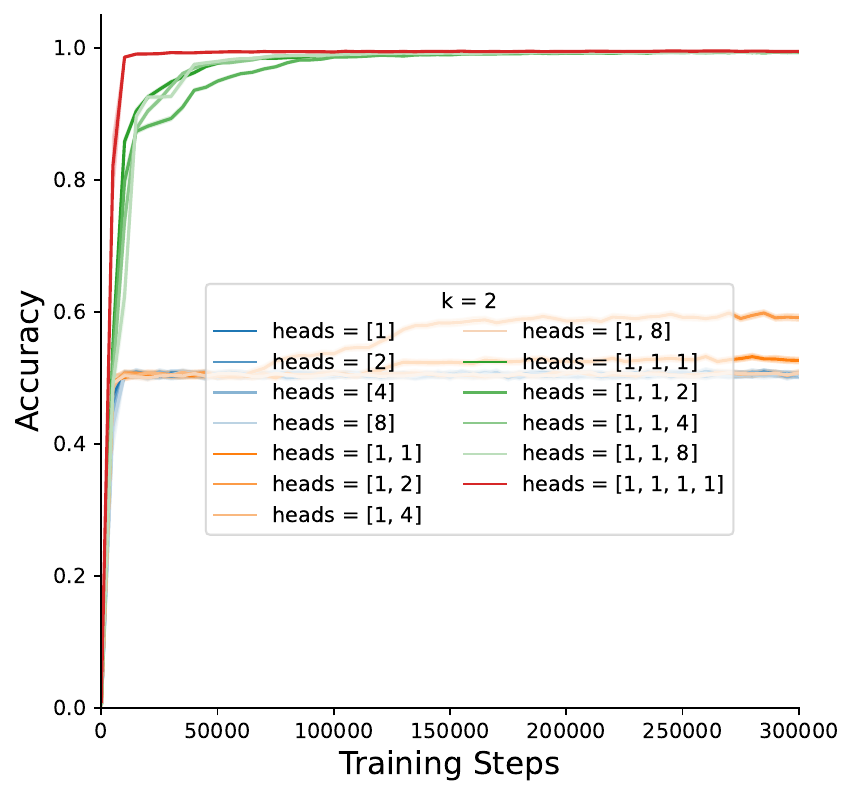}
\end{subfigure}
\begin{subfigure}{0.4\columnwidth}
    \centering
    \includegraphics[width=0.7\columnwidth]{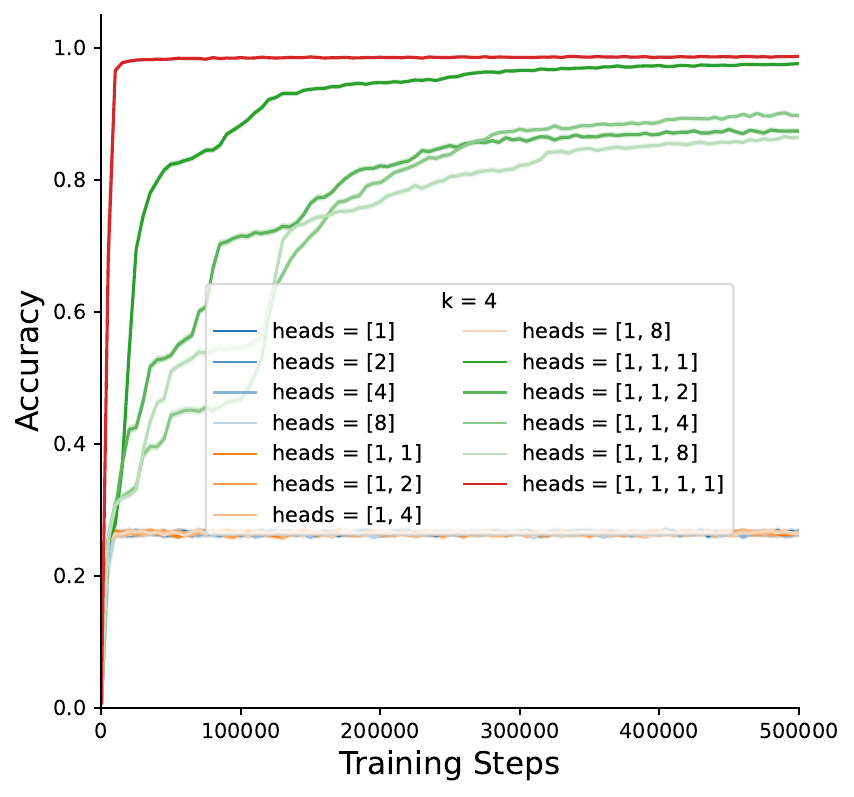}
\end{subfigure}
\begin{subfigure}{0.4\columnwidth}
    \centering
    \includegraphics[width=0.7\columnwidth]{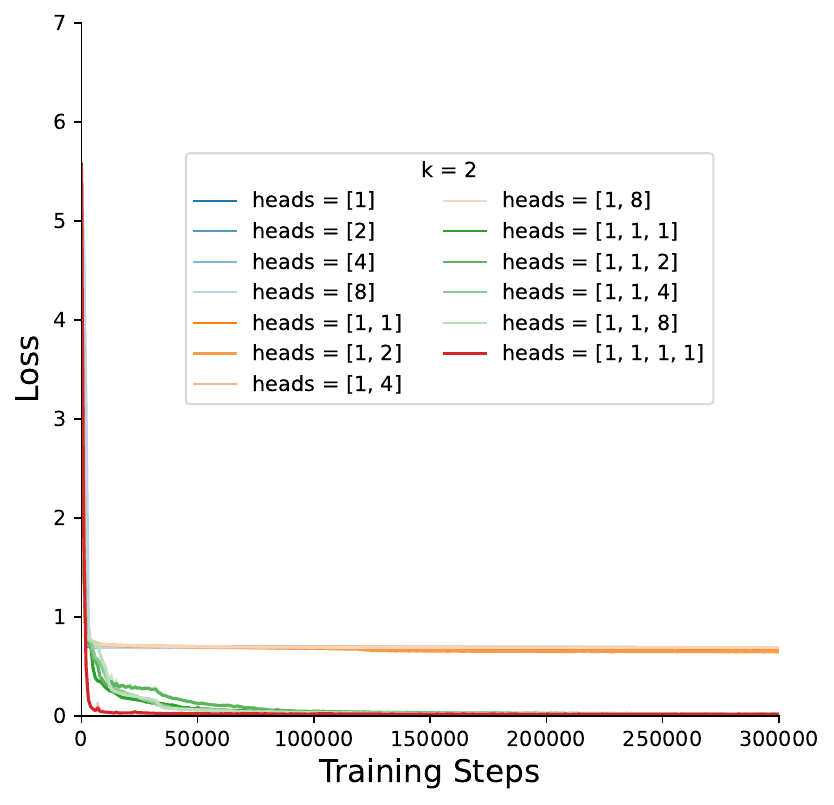}
\end{subfigure}
\begin{subfigure}{0.4\columnwidth}
    \centering
    \includegraphics[width=0.7\columnwidth]{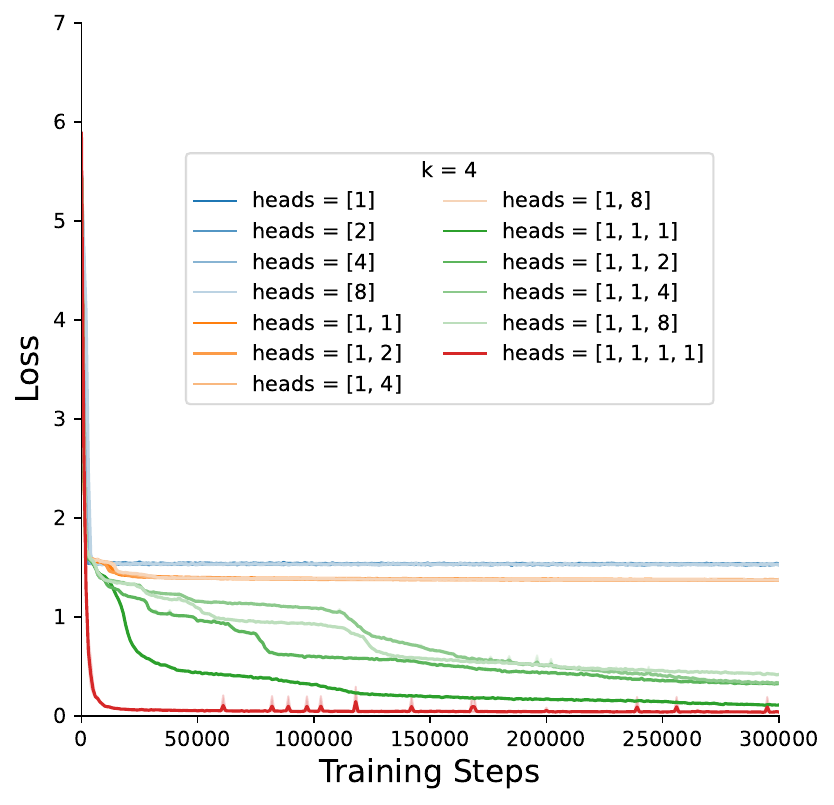}
\end{subfigure}
\caption{Training dynamic for different attention heads and layers on contextual generalization task on given $k$ examples}
\label{fig:add_context_tmplt_head}
\end{center}
\end{figure}

As shown in \cref{fig:add_context_tmplt_head}, similar to the results in  \cref{fig:add_icl_head} and \cref{fig:add_tmplt_head}, the addition of attention heads does not have a significant effect for generalization and reasoning. It is worth noting that when each sequence contains 4 examples, a 3-layer transformer with a single head in the last layer exhibits faster learning compared to models with multiple attention heads. This can be attributed to the inherent nature of larger models, which typically require more steps to converge \cite{bottou2018optimization}. Interestingly, we observed that a 4-layer transformer with a single head achieves even faster learning in the contextual generalization task compared to the 3-layer transformer. This suggests that by increasing the number of layers, the model can employ more complex and efficient methods to solve problems.

\section{Useful Transformer Constructions}

In this section, we introduce two useful constructions that can help us understand the linear attention model and make the following section of our construction clearer. The \textit{instructive attention} enables us to generate any token-invariant attention map $\boldsymbol{\alpha} $,  i.e., the attention map $\boldsymbol{\alpha} \in \mathbb{R}^{n\times n}_{+} $ is independent of the token embedding, but instead only relies on the position in the sequence. This implies that we can instruct the model to focus on the specific areas for each position. The \textit{constrained attention} allows us to apply custom masks to the original attention, restricting the attention of each token to a specific segment of the sequence. This will be applied as the core to make transformer generalization more effective.

\begin{lemma}[Instructive attention]
\label{lem:ins_attn} There exists an attention layer that can guide the attention between each token using the positional encoding. Let $\mathbf{H} \in \mathbb{R}^{d' \times n}$ be the input, represented as:

$$
\mathbf{H} =\begin{NiceArray}{\left\lgroup c \right\rgroup}
  \mathbf{X} \\
  \mathbf{P}\\
  \mathbf{0}\\ 
\end{NiceArray} =\begin{NiceArray}{\left\lgroup ccc \right\rgroup l}
  \mathbf{x}_0 & \cdots & \mathbf{x}_{n-1} & \rightarrow d \times n \\
  \mathbf{p}_0 & \cdots & \mathbf{p}_{n-1} & \rightarrow n \times n \\
  \mathbf{0} & \cdots & \mathbf{0} & \rightarrow (d' - n) \times n \\
\end{NiceArray}.
$$

For any attention score $\boldsymbol{\alpha} \in \mathbb{R}^{n\times n}_{+}$ that is independent of the input sequences, there always exists a weight matrix $\mathbf{W}_{QK} \in \mathbb{R}^{d' \times d'}$ such that the following holds for any input sequence $(\mathbf{x}_0, \cdots, \mathbf{x}_{n-1}) \in \mathcal{X}^n$:

$$
\sigma\left(\mathbf{H}^\top \mathbf{W}_{QK} \mathbf{H} \right) = \boldsymbol{\alpha}
$$
\end{lemma}
\begin{proof} 
Recall that the sequence will first be preprocessed into a $d' \times n$ matrix with positional encoding, and padding zero to make the input dimension align with the model's hidden size $d'$, result in the final input $\mathbf{H}$. Besides, note that we consider the one-hot positional encoding, i.e., $\mathbf{p}_i = [\underbrace{0,\dots,0}_{i},1,\underbrace{0,\dots,0}_{n-(i+1)}]^\top$. Then, we can first consider the attention weight matrix $\mathbf{W}_{QK} $ with the following form:

$$
\mathbf{W}_{QK} = \begin{NiceArray}{\left\lgroup ccc \right\rgroup}
  \mathbf{W}_{tt} & \mathbf{W}_{tp} & \mathbf{0}_{ d \times  (d' - n)}\\
  \mathbf{W}_{pt} & \mathbf{W}_{pp} & \mathbf{0}_{n \times  (d' - n)}\\
  \mathbf{0}_{(d' - n)  \times  d } & \mathbf{0}_{(d' - n)  \times  n} & \mathbf{0}_{(d' - n)  \times  (d' - n)} \\
\end{NiceArray}.
$$

Expanding the equation $\mathbf{H}^\top \mathbf{W}_{QK} \mathbf{H}$ gives:
\begin{equation}\label{eq:attn_expand}
\mathbf{H}^\top \mathbf{W}_{QK} \mathbf{H} = \mathbf{X}^\top \mathbf{W}_{tt} \mathbf{X} + \mathbf{P}^\top \mathbf{W}_{pt} \mathbf{X}  + \mathbf{X}^\top \mathbf{W}_{tp} \mathbf{P} + \mathbf{P}^\top \mathbf{W}_{pp} \mathbf{P}.
\end{equation}

Note that $\mathbf{P} = [\mathbf{p}_0 , \cdots , \mathbf{p}_{n-1} ] = \mathbf{I}$. Therefore, we can further set $\mathbf{W}_{tt} = \mathbf{0}$, $\mathbf{W}_{pt} = \mathbf{0}$, and $\mathbf{W}_{tp} = \mathbf{0}$, then

$$
\sigma\left(\mathbf{H}^\top \mathbf{W}_{QK} \mathbf{H} \right) = \sigma\left(\mathbf{P}^\top \mathbf{W}_{pp} \mathbf{P} \right) = \sigma\left(\mathbf{W}_{pp}\right).
$$

It is then clear that the attention map only depends on $\mathbf{W}_{pp}$ and has nothing to do with the token embeddings $\mathbf X$. Therefore, we can directly set  $\mathbf{W}_{pp} = \boldsymbol{\alpha}$ to complete the proof.
\end{proof}

\cref{lem:ins_attn} shows that the transformer can instruct the attention through positional encoding. That is, the transformer can aggregate any tokens together for further processing, which is essential for the instruction head where multiple tokens need to be copied together for further processing \cite{olsson2022context}.

Furthermore, applying the above construction can further help  implement attention masking, which aims to restrict attention between any specific pair of tokens based on their positions:

\begin{lemma}[Constrained attention]
\label{lem:con_attn} There exists an attention layer that can control the attention between any pair of tokens through the one-hot positional encoding, i.e., masking the correlation of tokens between any two positions.
\end{lemma}
\begin{proof} 
By \cref{eq:attn_expand}, we have

$$
\mathbf{H}^\top \mathbf{W}_{QK} \mathbf{H} = (\mathbf{X}^\top \mathbf{W}_{tt} \mathbf{X} + \mathbf{P}^\top\mathbf{W}_{pt} \mathbf{X}  + \mathbf{X}^\top \mathbf{W}_{tp}\mathbf{P} ) +  \mathbf{W}_{pp}.
$$

The first term on the R.H.S. of the above equation enables the model to learn attention between each token at different positions, while the right part $\mathbf{W}_{pp}$ can be arbitrarily designed. Therefore, in order to mask the attention between the tokens at positions $i$ and $j$, we can directly set $\mathbf W_{pp}(i,j)\rightarrow\infty$ or some sufficiently large negative value. In this way, we can constrain the attention between each token based on their position.
\end{proof}

In the following sections, these two attention methods will be useful to help achieve the desired task. In particular, instructive attention to aggregate tokens at specific positions, and utilize the constrained attention mechanism when the model needs to concentrate on a particular segment of the sequence.

\section{Limitation for Single Layer Attention-only Transformer}\label{sec:limitation_dependency}

In this section, we present an intriguing scenario where the input sequences possess specific properties. Under these conditions, the predictions made by a single layer transformer exhibit \textbf{linear dependence}, implying that the sequence labels must satisfy specific constraints. This limitation hinders the performance of the 1-layer transformer, and we will leverage this property to demonstrate why the single layer transformer struggles with reasoning and generalization in \cref{sec:prf_lsa_fail} and \cref{sec:prf_tmplt_fail}.

To begin, let us define a combination operation for discrete tokens and sequences, which is useful in describing the relationship among sequences.

\begin{definition}
\label{def:sequence_combination}
(\textit{Combination operation for discrete tokens and sequences}). Let $\mathbf{x}$ and $\mathbf{y}$ be two tokens chosen from the vocabulary $\mathcal{X}$. The combination operation for these two tokens, denoted as $\mathbf{x} \oplus \mathbf{y}$, is a multiset containing both tokens, i.e., $\mathbf{x} \oplus \mathbf{y} = [\mathbf{x}, \mathbf{y}]$. Furthermore, the combination between a multiset $ [\mathbf{x}, \mathbf{y}]$ and a token $\mathbf{z}$ results in a new multiset $[\mathbf{x}, \mathbf{y}, \mathbf{z}]$. For two sequences $\mathbf{X} = (\mathbf{x}_0,\dots,\mathbf{x}_{n-1})$ and $\mathbf{Y} = (\mathbf{y}_0,\dots,\mathbf{y}_{n-1})$, their combination $\mathbf{X} \oplus \mathbf{Y} = (\mathbf{x}_0 \oplus \mathbf{y}_0, \dots, \mathbf{x}_{n-1} \oplus \mathbf{y}_{n-1})$ is a sequence of multisets. Notably, for any positive integer $\lambda \in \mathbb{Z}_{+}$, $\lambda \otimes \mathbf{x}$ can be interpreted as a multiset consisting of $\lambda$ copies of token $\mathbf{x}$, i.e., $\{ \underbrace{\mathbf{x},\dots,\mathbf{x}}_{k \text{ times}} \}$. Similarly, $\lambda \otimes \mathbf{X} = (\lambda \otimes \mathbf{x}_0,\dots,\lambda \otimes\mathbf{x}_{n-1})$.
\end{definition}

\begin{definition}
\label{def:data_dependence}
(\textit{Sequence dependent}). We define the input sequences $\mathbf{X}^{(0)},\mathbf{X}^{(1)},\dots, \mathbf{X}^{(N-1)} \in \mathcal{X}^n$ as \textit{dependent} if they satisfy the following conditions:

\begin{enumerate}
    \item All sequences end with the same token, i.e., $\mathbf{X}^{(0)}[n-1] = \mathbf{X}^{(1)}[n-1] = \dots = \mathbf{X}^{(N-1)}[n-1]$.
    \item There exist coefficients $\lambda_0,\lambda_1,\dots,\lambda_{N-1} \in \mathbb{Z}$ (not all zero) such that 
    \begin{equation}
    \label{eq:data_dependence}
        \sqcap_{i \in \mathcal{I}_{+}} (\lambda_i \otimes \mathbf{X}^{(i)}) = \sqcap_{i \in \mathcal{I}_{-}} (-\lambda_i \otimes \mathbf{X}^{(i)}),
    \end{equation}
    where $\mathcal{I}_{+} = \{i |\lambda_{i} \geq 0, i \in [N]\}$ and $\mathcal{I}_{-} = \{i |\lambda_{i} < 0, i \in [N]\}$ divide the sequences into two parts. Here, $\sqcap_{i \in \mathcal{I} } \lambda_i \otimes \mathbf{X}^{(i)} = (\lambda_{i_1} \otimes \mathbf{X}^{(i_1)}) \oplus \dots \oplus (\lambda_{i_k} \otimes \mathbf{X}^{(i_k)})$ represents the operation of combining sequences.
\end{enumerate}
\end{definition}

\paragraph{Example} Consider a vocabulary $\mathcal{X} = \{\mathbf{a},\mathbf{b}, \mathbf{r}\}$ with $3$ discrete tokens. In this case, the sequences $(\mathbf{a},\mathbf{b},\mathbf{r}), (\mathbf{b},\mathbf{a},\mathbf{r}), (\mathbf{a},\mathbf{a},\mathbf{r}), (\mathbf{b},\mathbf{b},\mathbf{r})$ are dependent. This is because $(\mathbf{a},\mathbf{b},\mathbf{r}) \oplus (\mathbf{b},\mathbf{a},\mathbf{r}) = ([\mathbf{a},\mathbf{b}],[\mathbf{a},\mathbf{b}],[\mathbf{r},\mathbf{r}]) = (\mathbf{a},\mathbf{a},\mathbf{r}) \oplus (\mathbf{b},\mathbf{b},\mathbf{r})$. Additionally, we can observe that $\sum_{i=0}^{N-1} \lambda_i = 0$ since both sides of \cref{eq:data_dependence} should have the same number of discrete tokens in each position.

\begin{proposition}
\label{prop:out_linear}
If the input sequences $\mathbf{X}^{(0)},\mathbf{X}^{(1)},\dots, \mathbf{X}^{(N-1)} \in \mathcal{X}^n$  are \textbf{dependent}, then for the first attention layer $\texttt{TF}_1$ for any attention only transformer, we have

 $$
 \lambda_0 (\texttt{TF}_1(\mathbf{X}^{(0)})[n-1])  + \lambda_1 (\texttt{TF}_1(\mathbf{X}^{(1)})[n-1]) + \dots + \lambda_{N-1} (\texttt{TF}_1(\mathbf{X}^{(N-1)})[n-1]) = \mathbf{0}.
 $$

Here, $\{\lambda_i\}_{i=0}^{N-1}$ represents the coefficients defined in \cref{def:data_dependence}. We highlight the linear dependency in the output of at position $n-1$, which can be regarded as the representation for the next token or classification prediction of the entire sequence. 

\end{proposition}

\begin{proof}

Assuming the input sequence is $\mathbf{X} = [\mathbf{x}_0,\mathbf{x}_1,\dots,\mathbf{x}_{n-1}] \in \mathbb{R}^{d \times n}$, the 1-layer linear attention only transformer $\texttt{TF}_1$ performs a sequence-to-sequence mapping by first computing the attention score between each token and then aggregating them to obtain the final representation at each position:

\begin{equation}\label{eq:1_lsa}
\begin{split}
\texttt{TF}_1(\mathbf{X}) &= \mathbf{X} + \frac{1}{m} \sum_{i = 1}^{m}\left[ (\mathbf{W}_{V_i} \mathbf{X} ) \sigma((\mathbf{W}_{Q_i} \mathbf{X} )^\top  (\mathbf{W}_{K_i} \mathbf{X}) )\right] \\
                        &= \mathbf{X} + \frac{1}{m} \sum_{i = 1}^{m}\left[  (\mathbf{W}_{V_i}  \mathbf{X}) \sigma(\mathbf{X}^\top  \mathbf{W}_{QK_i} \mathbf{X} )\right] 
\end{split}
\end{equation}

Here, we have combined the learnable parameters $\mathbf{W}_{Q_i},\mathbf{W}_{K_i} \in \mathbb{R}^{h \times d}$ into a single matrix $\mathbf{W}_{QK_i} \in \mathbb{R}^{d \times d}$ for analysis simplification.

Consider the output at position $k \in [n]$:

\begin{equation}\label{eq:1_lsa_k}
\begin{split}
\texttt{TF}_1(\mathbf{X})[k]  &= \mathbf{x}_k  + \frac{1}{m} \sum_{i = 1}^{m}\left[(\mathbf{W}_{V_i} \mathbf{X} )\sigma(\mathbf{X}^\top  \mathbf{W}_{QK_i} \mathbf{x}_k )  \right] \\
                            &= \mathbf{x}_k  + \frac{1}{m} \sum_{i = 1}^{m} \sum_{j = 0}^{n-1} \left[ (\mathbf{W}_{V_i} \mathbf{x}_j ) \sigma(\mathbf{x}_j^\top  \mathbf{W}_{QK_i} \mathbf{x}_k ) \right] \\
                            &= \mathbf{x}_k  + \frac{1}{m} \sum_{j = 0}^{n-1} \sum_{i = 1}^{m} \left[ (\mathbf{W}_{V_i} \mathbf{x}_j ) \sigma(\mathbf{x}_j^\top  \mathbf{W}_{QK_i} \mathbf{x}_k ) \right] ,
\end{split}
\end{equation}

we define $\texttt{Attn}(\mathbf{x}_j,\mathbf{x}_k) := \frac{1}{m}\sum_{i = 1}^{m} \left[ (\mathbf{W}_{V_i} \mathbf{x}_j ) \sigma(\mathbf{x}_j^\top  \mathbf{W}_{QK_i} \mathbf{x}_k ) \right] $, in this way the output at position $k$ can be written as $\texttt{TF}(\mathbf{X})[k] = \mathbf{x}_k  +  \sum_{j = 1}^{n} \texttt{Attn}(\mathbf{x}_j,\mathbf{x}_k)$. Specifically, when the last token for all input sequences is the same, the prediction based on the last token is:

\begin{equation}\label{eq:1_lsa_n}
\begin{split}
\texttt{TF}_1(\mathbf{X})[n-1]  &= \mathbf{x}_{n-1}  +  \sum_{j = 0}^{n-1} \texttt{Attn}(\mathbf{x}_j,\mathbf{x}_{n-1}).
\end{split}
\end{equation}

If the input sequences are dependent, based on \cref{def:data_dependence}, we can divide the sequences into two groups $\mathcal{I}_{+}$ and $\mathcal{I}_{-}$ so in each position $j$, both side have the same occurrence for each token:

\begin{equation}
\sqcap_{i \in \mathcal{I}_{+}} (\lambda_i \otimes \mathbf{x}^{(i)}_j) = \sqcap_{i \in \mathcal{I}_{-}} (-\lambda_i \otimes \mathbf{x}^{(i)}_j) := \mathcal{S}_j,
\end{equation}

here we use $\mathcal{S}_j$ to denote the tokens occurrences at position $j$, note that $\lambda \texttt{Attn}(\mathbf{x}^{(i)}_j,\mathbf{x}_{n-1}) = \sum_{\mathbf{s} \in (\lambda \otimes \mathbf{x}^{(i)}_j)} \texttt{Attn}(\mathbf{s},\mathbf{x}_{n-1})$, so we can derive the following equation:

\begin{equation}\label{eq:depen_data}
\sum_{i \in \mathcal{I}_{+}} \lambda_i \texttt{Attn}(\mathbf{x}^{(i)}_j,\mathbf{x}_{n-1}) = \sum_{\mathbf{s} \in \mathcal{S}_j} \texttt{Attn}(\mathbf{s},\mathbf{x}_{n-1}) = \sum_{i \in \mathcal{I}_{-}} -\lambda_i \texttt{Attn}(\mathbf{x}^{(i)}_j,\mathbf{x}_{n-1}).
\end{equation}

Then for the final prediction result

\begin{equation}\label{eq:1_lsa_n_N}
\begin{split}
 &\lambda_0 (\texttt{TF}_1(\mathbf{X}^{(0)})[n-1])  + \lambda_1 (\texttt{TF}_1(\mathbf{X}^{(1)})[n-1]) + \dots + \lambda_{N-1} (\texttt{TF}_1(\mathbf{X}^{(N-1)})[n-1]) \\
 =& \left(\mathbf{x}_{n-1} \sum_{i=0}^{N-1} \lambda_i \right) + \left(\sum_{j = 0}^{n-1} \sum_{i=0}^{N-1} \lambda_i \texttt{Attn}(\mathbf{x}^{(i)}_j,\mathbf{x}_{n-1})\right) \\
 =& \mathbf{0}.
 \end{split}
\end{equation}

The first part of \cref{eq:1_lsa_n_N} is zero based on the observation that $\sum_{i=0}^{N-1} \lambda_i = 0$ since both sides of \cref{eq:data_dependence} should have the same number of discrete tokens in each position, and the second part is zero based on \cref{eq:depen_data},  which states that each token should have the same occurrence for both groups.
\end{proof}

In a single layer transformer, which consists of only one attention layer and a linear classifier layer, the result $\mathbf{o} = \mathbf{W}_{O} \texttt{TF}_1(\mathbf{X})[n-1])$ also satisfies the linear dependency:

\begin{proposition}
\label{prop:pre_linear}
If the input sequences $\mathbf{X}^{(0)},\mathbf{X}^{(1)},\dots, \mathbf{X}^{(N-1)} \in \mathcal{X}^n$  are \textbf{dependent}, then for any single layer transformer $\texttt{TF}$, their prediction result $\mathbf{o}^{(0)},\dots,\mathbf{o}^{(N-1)}$
 $$
 \lambda_0 \mathbf{o}^{(0)}  + \lambda_1 \mathbf{o}^{(1)} + \dots + \lambda_{N-1} \mathbf{o}^{(N-1)} = \mathbf{0}.
 $$
 Here, $\{\lambda_i\}_{i=0}^{N-1}$ represents the coefficients defined in \cref{def:data_dependence}. 
\end{proposition}

\begin{proof}
    By directly applying \cref{prop:out_linear} and multiplying both sides by $\mathbf{W}_{O}$, we have:
    
    $$
 \lambda_0 \mathbf{W}_{O} (\texttt{TF}_1(\mathbf{X}^{(0)})[n-1])  + \lambda_1 \mathbf{W}_{O}(\texttt{TF}_1(\mathbf{X}^{(1)})[n-1]) + \dots + \lambda_{N-1} \mathbf{W}_{O}(\texttt{TF}_1(\mathbf{X}^{(N-1)})[n-1]) = \mathbf{W}_{O} \mathbf{0}.
 $$
In a single layer transformer, the prediction is conducted by a linear projection $\mathbf{W}_{O} \in \mathbb{R}^{c \times d'}$. Therefore, we can replace $\mathbf{W}_{O} \texttt{TF}_1(\mathbf{X}^{(i)}[n-1])$ with $\mathbf{o}^{(i)}$, thus completing the proof.
\end{proof}

According to \cref{prop:pre_linear}, the prediction of a single layer transformer for certain dependent sequences should be linearly dependent. For instance, let's consider the four sequences: $(\mathbf{a},\mathbf{b},\mathbf{r}), (\mathbf{b},\mathbf{a},\mathbf{r}), (\mathbf{a},\mathbf{a},\mathbf{r}), (\mathbf{b},\mathbf{b},\mathbf{r})$. In this case, the final predictions $\mathbf{o}_0,\mathbf{o}_1,\mathbf{o}_2,\mathbf{o}_3$ will satisfy the equation $\mathbf{o}_0 + \mathbf{o}_1 = \mathbf{o}_2 + \mathbf{o}_3$ for any single layer transformer. This observation implies that the labels of these sequences are restricted for a single layer transformer to fit all of them. We will explain later in \cref{sec:prf_lsa_fail} and \cref{sec:prf_tmplt_fail} why this leads to the failure of the single layer transformer in reasoning and generalization. Besides, we extend \cref{prop:pre_linear} to from ReLU attention to softmax attention for single layer single head transformer in \cref{sec:soft_max}.

\section{Proofs for \cref{sec:mem_main}}
\label{sec:prof_mem}
In this section, we will prove the results in  \cref{sec:mem_main}, which asserts that a single-layer transformer with a sufficient number of attention heads can effectively memorize sequences with distinct labels. We establish this by first proving the existence of a linear projection layer that can map any non-parallel vectors with distinct labels. Subsequently, we provide evidence that with a sufficient number of attention heads, there exists a transformer can map any sequences to non-parallel vectors. 

In the following section, we will first show that there exists a single-layer transformer with $n$ heads, $d' = \max\{nd, d + n\}$, that any input sequence $\mathbf{H} \in\mathbb{R}^{d' \times n}$ can be mapped to a corresponding distinct label so that the memorizing task can be perfectly performed, our construction procedure can be formulated as follows:

$$
\mathbf{H} = \begin{NiceArray}{\left\lgroup cccc \right\rgroup }
  \mathbf{x}_0 &   \cdots & \mathbf{x}_{k-1} & \mathbf{0} \\
  0 &   \cdots & 0 & 1 \\
  \mathbf{p}_0 & \cdots & \mathbf{p}_{k-1} & \mathbf{p}_{k } \\
  \mathbf{0}& \cdots & \mathbf{0} & \mathbf{0}  \\
\CodeAfter
\OverBrace[shorten,yshift=3pt]{1-1}{1-3}{\scriptstyle{\substack{sequence \\ \textit{with } k \textit{ tokens}}}}
\OverBrace[shorten,yshift=3pt]{1-4}{1-4}{\scriptstyle{\substack{response \\ sign}}} \
\end{NiceArray} \xrightarrow[\cref{lem:seq_mapping}] {\texttt{TF}, \text{ last token}} \begin{NiceArray}{\left\lgroup c \right\rgroup }
  \mathbf{x}_0 \\
  \vdots \\
  \mathbf{x}_{k-1} \\
  \mathbf{0}\\
\end{NiceArray}  \xrightarrow[\cref{lem:vector_mapping}]{\text{linear classifier}} \begin{NiceArray}{\left\lgroup c \right\rgroup l}
  o_0 \\
  \vdots \\
  o_j & \scriptstyle{\substack{\text{max element} \\ (\text{distinct } j \text{ for each sequence})}} \\ 
  \vdots \\
  o_{C-1} \\
\end{NiceArray}
$$

We first deliver the following lemma, which proves that the linear classify can perfectly map the representation of any input sequence to the desired labels.
\begin{lemma}
\label{lem:vector_mapping}
Let $\{(\mathbf{x}_0,y_0),\dots,(\mathbf{x}_{N-1},y_{N-1})\} \subset \mathbb{R}^{d \times N} \times [N]^{N}$ be a dataset of $N$ vector-label pairs, where every two vectors $\mathbf{x}_i,\mathbf{x}_j$ are linearly independent, and the labels of each vector are different. Then, there exists a linear projector $\mathbf{W} \in \mathbb{R}^{d \times N}$ such that:
    $$
    (\mathbf{x}_i \mathbf{W} )[y_i] > (\mathbf{x}_i \mathbf{W} )[j] \quad \forall j \in [N], j \neq y_i, \forall i \in [N].
    $$
\end{lemma}

\begin{proof}
    
Considering the result of the linear projector as the prediction of the probability of the sequence belonging to each class $1,\dots,N$, it always predicts the highest possibility at the index of the corresponding label. In this case, the model will achieve a perfect accuracy rate. We construct $\mathbf{W} = [\mathbf{w}_0,\mathbf{w}_1,\dots,\mathbf{w}_{N-1}]$ by assigning the value $\mathbf{w}$ at each column one by one. WLOG, we let the data be $
\{(\mathbf{x}_0,0),(\mathbf{x}_1,1),\dots,(\mathbf{x}_{N-1},{N-1})\} 
$.

We started by consider the case $N = 2$, because $\mathbf{x}_0,\mathbf{x}_1$ are linearly independent, for any $\mathbf{a},\mathbf{b} \in \mathbb{R}^{2}$, \cref{eq:2d_solution} have at least one solution for $\begin{bmatrix}
\mathbf{w}_0 & \mathbf{w}_1
\end{bmatrix}$.

\begin{equation}\label{eq:2d_solution}
\begin{NiceArray}{\left\lgroup c \right\rgroup}
    \mathbf{x}_0\\
    \mathbf{x}_1
\end{NiceArray} \begin{NiceArray}{\left\lgroup cc \right\rgroup}
  \mathbf{w}_0 & \mathbf{w}_1
\end{NiceArray} = \begin{NiceArray}{\left\lgroup c \right\rgroup}
    \mathbf{a}\\
    \mathbf{b}\\
\end{NiceArray},
\end{equation}

we can let $\mathbf{a} = [1,0],\mathbf{b} = [0,1]$ and solve the above equation, the corresponding solution $\mathbf{w}_0,\mathbf{w}_1$ can correctly classify $\{ (\mathbf{x}_0,0),(\mathbf{x}_1,1) \}$.

If $N = k$ satisfies the above condition, which means that there exists $[\mathbf{w}_0,\mathbf{w}_1,\dots,\mathbf{w}_{k-1}]$ that can correctly classify $\{ (\mathbf{x}_0,0),(\mathbf{x}_1,1),\dots,(\mathbf{x}_{k-1},{k-1}) \}$, then consider a new data point $(\mathbf{x}_{k},k)$. We construct $\mathbf{w}_{k}$ as follows:

First, we compute the result $\mathbf{o}^{(k)}_{:k} = \mathbf{x}_{k} [\mathbf{w}_0,\mathbf{w}_1,\dots,\mathbf{w}_{k-1}]$ and choose the highest index $i_{k} = \text{argmax}_{i} \mathbf{o}^{(k)}_{:k}$. Then, we consider the result if we set $\mathbf{w}_{k} = \mathbf{w}_{i_k} $:
$$
\begin{NiceArray}{\left\lgroup c \right\rgroup}
    \mathbf{x}_0\\
\vdots\\
\mathbf{x}_{i_k}\\
\vdots\\
\mathbf{x}_{k}
\end{NiceArray}\begin{NiceArray}{\left\lgroup cccccc \right\rgroup}
\mathbf{w}_0 & \dots & \mathbf{w}_{i_k} & \dots & \mathbf{w}_{k-1} & \mathbf{w}_{i_k} \\
\CodeAfter
\OverBrace[shorten,yshift=3pt]{1-1}{1-5}{\scriptstyle{k}}
\UnderBrace[shorten,yshift=3pt]{1-6}{1-6}{\scriptstyle{\substack{\text{init } \\ \mathbf{w}_{k}= \mathbf{w}_{i_k} }}}
\end{NiceArray}
= \begin{NiceArray}{\left\lgroup cccccc \right\rgroup}
o_{o,o}&\dots& o_{0,i_k} & \dots& o_{0,k-1} & o_{0,i_k}\\
\vdots & \ddots & \vdots & \ddots & \vdots & \vdots\\
o_{i_k,0}&\dots &o_{i_k,i_k} & \dots& o_{i_k,k-1} & o_{i_k,i_k}\\
\vdots & \ddots & \vdots & \ddots & \vdots & \vdots\\
o_{k-1,0}&\dots &o_{k-1,i_k} & \dots& o_{k-1,k-1} & o_{k-1,i_k}\\
o_{k,0}&\dots& o_{k,i_k} & \dots& o_{k,k-1} & o_{k,i_k}\\
\end{NiceArray},
$$

We can see that appending $\mathbf{w}_{i_k}$ wouldn't change the prediction result for $\{ \mathbf{x}_0,\dots,\mathbf{x}_{i_k - 1},\mathbf{x}_{i_k + 1},\dots,\mathbf{x}_{k-1} \}$. Therefore, we just need to find an extra $\Delta  \mathbf{w}_{k}$ such that $\mathbf{w}_{k} = \mathbf{w}_{i_k} + \Delta  \mathbf{w}_{k}$ remains the prediction for $\{ \mathbf{x}_0,\dots,\mathbf{x}_{i_k - 1},\mathbf{x}_{i_k + 1},\dots,\mathbf{x}_k \}$ unchanged (\cref{eq:old_class}), and can successfully classify $(\mathbf{x}_{i_k} ,i_k)$ and $(\mathbf{x}_{k}, k)$ (\cref{eq:new_class}).

\begin{equation}\label{eq:old_class}
\begin{NiceArray}{\left\lgroup cccccc \right\rgroup}
\mathbf{x}_0 & \dots & \mathbf{x}_{i_k - 1} & \mathbf{x}_{i_k + 1} & \cdots & \mathbf{x}_{k-1}
\end{NiceArray}^\top
\Delta \mathbf{w}_{k} \prec \begin{NiceArray}{\left\lgroup cccccc \right\rgroup}
\epsilon & \dots & \epsilon & \epsilon & \cdots & \epsilon
\end{NiceArray}^\top,
\end{equation}

\begin{equation}\label{eq:new_class}
\begin{cases}
\mathbf{x}_{i_k} \Delta \mathbf{w}_{k} &< 0 \\
\mathbf{x}_{k} \Delta \mathbf{w}_{k} &> 0
\end{cases},
\end{equation}

where $\epsilon = \min_{i \in [k]_{\nmid i_k}}\{p_{i,i} - p_{i,i_k}\}$, which is also a maximum vibration that can remain the prediction result for $\{ \mathbf{x}_0,\dots,\mathbf{x}_{i_k - 1},\mathbf{x}_{i_k + 1},\dots,\mathbf{x}_{k-1} \}$ unchanged. Such $\Delta \mathbf{w}_{k}$ exists as $\mathbf{x}_{i_k}$ and $\mathbf{x}_{k}$ are linearly independent. Therefore, we can first solve the equation 
\begin{equation}
\begin{cases}
\mathbf{x}_{i_k} \Delta \mathbf{w}_{k} &= -1\\
\mathbf{x}_{k} \Delta \mathbf{w}_{k} &= 1
\end{cases} ,
\end{equation}

and then rescale the $\Delta \mathbf{w}_{k} = \frac{1}{M} \Delta \mathbf{w}_{k}$ to ensure that the vibration for any other rows is less than $\epsilon$. Finally, we assign $\mathbf{w}_{k}  = \Delta \mathbf{w}_{k} + \mathbf{w}_{i_k} $, which ensures that the prediction result for $\mathbf{x}_{k}$ is $k$, while the prediction for other rows remains unchanged.

\end{proof}
Given the mapping capability of sequences to non-parallel vectors using \cref{lem:vector_mapping}, it follows that constructing a classifier layer to map each vector to a distinct label becomes straightforward. Consequently, our focus turns to the second part, where we aim to establish the validity of \cref{lem:seq_mapping}. This lemma asserts that a single-layer transformer, equipped with $n$ attention heads and a hidden size of $\max\{kd, d + k\}$, can effectively map a sequence to a non-parallel vector.

\begin{lemma}
\label{lem:seq_mapping}
Given a vocabulary $\mathcal{X}$ of non-parallel vectors, then  there exists a single-layer transformer with $n$ attention heads and hidden size $d' = \max\{k\cdot d, d + k\}$ to map all possible sequences of length $n$ (ending with a response sign) to non-parallel vectors.
\end{lemma}
\begin{proof}

To achieve this, we construct a transformer with $n$ attention heads, where each head processes only the token at its corresponding position $i$, for the first $k$ heads, we process the sequence for the last token (response sign) as follows:

$$
\mathbf{W}_{V_i} \begin{NiceArray}{\left\lgroup cccc \right\rgroup }
  \mathbf{x}_0 &   \cdots & \mathbf{x}_{k-1} & \mathbf{0} \\
  0 &   \cdots & 0 & 1 \\
  \mathbf{p}_0 & \cdots & \mathbf{p}_{k-1} & \mathbf{p}_{k } \\
  \mathbf{0}& \cdots & \mathbf{0} & \mathbf{0}  \\
\end{NiceArray}, \quad \boldsymbol{\alpha}_{i}[:,n-1] = n \begin{NiceArray}{\left\lgroup c \right\rgroup }
  \mathbf{0}_{i\cdot d \times 1} \\
  \mathbf{x}_i \\
  \mathbf{0}_{(d' -(i+1) d) \times 1}\\
\end{NiceArray},
$$

where $\boldsymbol{\alpha}_{i} \in \mathbb{R}^{n \times n} $ is the attention map for the $i$-th head, we can assign $\boldsymbol{\alpha}_{i}[:,n-1] = [\underbrace{0,\dots,0}_{i},1,\underbrace{0,\dots,0}_{n-(i+1)}]^\top$ using the construction for instructive attention \cref{lem:ins_attn}, and we set $\mathbf{W}_{V_i}$ to copy the token at position $i$ in a disentangled manner: 

$$
\mathbf{W}_{V_i} = n  \begin{NiceArray}{\left\lgroup cc \right\rgroup}
\mathbf{0}_{i \cdot d \times d} & \mathbf{0}_{i \cdot d  \times (d' - d)} \\
\mathbf{I}_{d \times d} & \mathbf{0}_{d \times (d' - d)} \\
\mathbf{0}_{(d' -(i+1) d) \times d} & \mathbf{0}_{(d' -(i+1) d) \times (d' - d)} \\
\end{NiceArray}.
$$

Then we set the $k$-th head to neutralize the information from the residue: 

$$
\mathbf{W}_{V_k} = -n \mathbf{I} \quad \boldsymbol{\alpha}_{k} = \mathbf{I} .
$$

As a result, the processed token representation at the last position is:

\begin{equation}
\begin{split}
\mathbf{H}^{(1)}[n-1] &= \mathbf{H}^{(0)}[n-1] + \frac{1}{n} \sum_{i = 1}^{n} \mathbf{W}_{V_{i}}  \mathbf{H}^{(0)} \boldsymbol{\alpha}_i[:,n-1] \\
                      &= \mathbf{H}^{(0)}[n-1] \underbrace{- \mathbf{H}^{(0)}[n-1]}_{k\text{-th head}} +\underbrace{ \sum_{i = 0}^{k-1} \mathbf{W}_{V_{i}}  \mathbf{H}^{(0)} \boldsymbol{\alpha}_i[:,n-1] }_{0,\dots,k-1 \text{-head} }\\
                      &= \begin{NiceArray}{\left\lgroup cccc \right\rgroup }
  \mathbf{x}_0 &
  \cdots &
  \mathbf{x}_{k-1} &
  \mathbf{0}
  \end{NiceArray}^\top
 \end{split}.
\end{equation}

Since the tokens in the vocabulary $\mathcal{X}$ are non-parallel vectors, their concatenation should also be non-parallel for different tokens. Therefore, we have successfully constructed a single-layer transformer with $n$ heads that can map sequences to non-parallel vectors by concatenating the tokens based on their positional information.
\end{proof}
Based on the aforementioned derivation, we can construct our final single-layer transformer as follows: first, employ the attention weights from \cref{lem:seq_mapping} to reshape sequences into no-parallel vectors , and then employ the methodology from \cref{lem:vector_mapping} to construct a classifier layer that maps each vector to a distinct label.

\section{Proofs for \cref{sec:icl_main}}

In this section, we will prove the main results in \cref{sec:icl_main}, which shows that the transformer requires at least two layers to perform successful reasoning. We will first prove that the reasoning task can never be perfectly resolved using single-layer transformer, no matter how many heads are included. Then, we show that two-layer transformers are capable of performing the designed reasoning tasks with perfect accuracy, by proving the existence of a set of attention weight matrices.

\subsection{Proofs of \cref{prop:lsa_fail}}
\label{sec:prf_lsa_fail}
\begin{proof}

In order to prove the inability of the single-layer transformer, we will design a specific task and show that any single-layer transformer cannot achieve perfect test accuracy. In particular, We consider a task that contains  $2$ question $\mathbf{a},\mathbf{b}$ and $2$ answer $\mathbf{x},\mathbf{y}$, take the response sign as \texttt{"="}. Then, the dataset $\mathcal{D}_{\texttt{toy-icl}} = \{\mathbf{E}^{(0)},\mathbf{E}^{(1)},\mathbf{E}^{(2)},\mathbf{E}^{(3)}\}$ can be denote as follows:

\begin{center}
\begin{tcolorbox}[text width = 7cm]
\begin{enumerate}
    \item input : $\mathbf{a} = \mathbf{x}$ $\mathbf{b} = \textbf{y}$ $\mathbf{a} = $ target : $\textbf{x}$
    \item input : $\mathbf{a} = \textbf{y}$ $\mathbf{b} = \mathbf{x}$ $\mathbf{b} = $ target : $\textbf{x}$
    \item input : $\mathbf{a} = \textbf{y}$ $\mathbf{b} = \textbf{x}$ $\mathbf{a} = $ target : $\textbf{y}$
    \item input : $\mathbf{a} = \mathbf{x}$ $\mathbf{b} = \textbf{y}$ $\mathbf{b} = $ target : $\textbf{y}$
\end{enumerate}
\end{tcolorbox}
\end{center}

It can be shown that these 4 sequences are \textbf{dependent}, as we defined in \cref{def:data_dependence}.  To verify this, we first have that all these sequences end with the same token \texttt{"="}, then, it holds that

$$
\mathbf{E}^{(0)} \oplus \mathbf{E}^{(1)} = ([\mathbf{a},\mathbf{a}],[\texttt{"="},\texttt{"="}],[\mathbf{x},\mathbf{y}],[\mathbf{b},\mathbf{b}],[\texttt{"="},\texttt{"="}],[\mathbf{x},\mathbf{y}],[\mathbf{a},\mathbf{b}]) = \mathbf{E}^{(2)} \oplus \mathbf{E}^{(3)}.
$$

Then, for any single layer transformer \texttt{TF}, we can leverage \cref{prop:out_linear} and then get that the output representations for the last token corresponding to all sequences are linearly dependent, i.e.,

$$
\mathbf{o}^{(0)} + \mathbf{o}^{(1)} = \mathbf{o}^{(2)} + \mathbf{o}^{(3)}.
$$

Then we are ready to show that there doesn't exist a single-layer transformer that can reason all these $4$ examples. We will prove this by contradiction. First , suppose that there exists such a transformer \texttt{TF} that can correctly reason all these examples, then the output of the transformer will have the maximum output corresponding to the desired answer. Let $i_x$ and $i_y$ be the indices of the transformer output corresponding to the target $\mathbf x$ and $\mathbf y$ respectively, it shall hold that 

\begin{equation}\label{eq:toy_1}
\begin{split}
    \mathbf{o}^{(0)}[i_x] > \mathbf{o}^{(0)}[i_y],\quad
    \mathbf{o}^{(1)}[i_x] > \mathbf{o}^{(1)}[i_y],\quad
    \mathbf{o}^{(2)}[i_y] > \mathbf{o}^{(2)}[i_x],\quad
    \mathbf{o}^{(3)}[i_y] > \mathbf{o}^{(3)}[i_x].
\end{split}
\end{equation}

Besides, by linear dependency, we have
\begin{equation}\label{eq:toy_2}
\begin{split}
     \mathbf{o}^{(0)} -  \mathbf{o}^{(2)} =  \mathbf{o}^{(3)} -  \mathbf{o}^{(1)}.
\end{split}
\end{equation}
Combining \cref{eq:toy_1} and \cref{eq:toy_2}, we have the following contradiction:

\begin{equation}\label{eq:toy_3}
\begin{split}
    (\mathbf{o}^{(0)} - \mathbf{o}^{(2)})[i_x] >(\mathbf{o}^{(0)} - \mathbf{o}^{(2)})[i_y] > (\mathbf{o}^{(3)} - \mathbf{o}^{(1)})[i_y] >(\mathbf{o}^{(3)} - \mathbf{o}^{(1)})[i_x] >(\mathbf{o}^{(0)} - \mathbf{o}^{(2)})[i_x].
\end{split}
\end{equation}

Therefore, this implies that no single-layer transformer can correctly reason all of these four sequences. In other words, if these four sequences appear with equal probability, the reasoning accuracy achieved by any single-layer transformer will be upper bounded by $3/4$.

\end{proof}
\subsection{Proof for \cref{prop:lsa_succ}}
\label{sec:cons_icl}

In this section we will show that a 2-layer transformer can perfectly perform the reasoning tasks, when provided with proper weights. In particular,  we will  construct such a transformer by following a copy-matching process, that is, the first layer of transformer copies the answer to the corresponding question ahead of them, and then the second layer searches these question-answer pairs and chooses the one with the highest similarity, i.e. having the same token, then the classifier layer projects the representation embedding to the answer.

\begin{proof}
Recall the data construction process, we consider an input sequence with $k$ question-answer pairs, which is denoted as $\mathbf{H}^{(0)} \in \mathbb{R}^{d' \times n}$ ($n = 3 k + 2$):

$$
\mathbf{H}^{(0)} = 
\begin{NiceArray}{\left\lgroup ccccccccc \right\rgroup l}

  \mathbf{q}_0 & \mathbf{0} &  \mathbf{0} & \cdots & \mathbf{q}_{k-1} & \mathbf{0} &  \mathbf{0} &\mathbf{q}_r & \mathbf{0} & \rightarrow d_1 \times n \\
  \mathbf{0} & 1 &  \mathbf{0} & \cdots & \mathbf{0} & 1 &  \mathbf{0} & \mathbf{0} & 1 & \rightarrow 1 \times n \\
  \mathbf{0} & \mathbf{0} &  \mathbf{a}_0 & \cdots & \mathbf{0} & \mathbf{0} &  \mathbf{a}_{k-1} & \mathbf{0} & \mathbf{0} & \rightarrow d_2 \times n \\
  \mathbf{p}_0 & \mathbf{p}_1 &\mathbf{p}_2 &\cdots & \mathbf{p}_{3k - 3} & \mathbf{p}_{3k - 2} &\mathbf{p}_{3k-1} & \mathbf{p}_{3k} & \mathbf{p}_{3k+1} & \rightarrow n \times n \\
  \mathbf{0} & \mathbf{0} & \mathbf{0} & \cdots & \mathbf{0} & \mathbf{0} & \mathbf{0} &\mathbf{0} &\mathbf{0} & \rightarrow (d' - n - d) \times n \\
\CodeAfter
\OverBrace[shorten,yshift=3pt]{1-1}{1-3}{\scriptstyle{Q-R-A}},
\end{NiceArray}
$$
where $[\mathbf q_k; \mathbf{0}]$ and $[ \mathbf{0}; \mathbf{a}_k]$ denote the embeddings for the $k$-th question token and $k$-th answer token respectively, $[\mathbf{0};1;\mathbf{0}]$ denotes the embedding for the response sign, $r \in [k]$ is a random choose question index. Without loss of generality, we will set $r = 0$ in the following proof to illustrate how our construction works:

$$
\mathbf{H}^{(0)} = 
\begin{NiceArray}{\left\lgroup cccccc \right\rgroup }

  \mathbf{q}_0 & \mathbf{0} &  \mathbf{0} & \cdots & \mathbf{q}_0 & \mathbf{0} \\
  \mathbf{0} & 1 &  \mathbf{0} & \cdots & \mathbf{0} & 1\\
  \mathbf{0} & \mathbf{0} &  \mathbf{a}_0 & \cdots  & \mathbf{0} & \mathbf{0}\\
  \mathbf{p}_0 & \mathbf{p}_1 &\mathbf{p}_2 &\cdots & \mathbf{p}_{3k} & \mathbf{p}_{3k + 1} \\
  \mathbf{0} & \mathbf{0} & \mathbf{0} & \cdots &\mathbf{0} &\mathbf{0} \\
\CodeAfter
\OverBrace[shorten,yshift=3pt]{1-1}{1-3}{\scriptstyle{Q-R-A}}
\end{NiceArray} \xrightarrow[\substack{\text{copy among}\\ \text{Q-A block}}]{\texttt{TF}_1}  \begin{NiceArray}{\left\lgroup cccccc \right\rgroup }

  \mathbf{q}_0 & \mathbf{q}_0 &  \mathbf{q}_0 & \cdots & \mathbf{q}_0 & \mathbf{q}_0 \\
  1 & 1 &  1 & \cdots & 1 & 1\\
  \mathbf{a}_0  & \mathbf{a}_0 & \mathbf{a}_0 & \cdots  & \mathbf{0} & \mathbf{0}\\
  \vdots & \vdots &\vdots &\ddots & \vdots & \vdots
\CodeAfter
\OverBrace[shorten,yshift=3pt]{1-1}{1-1}{\scriptstyle{Q-R-A}}
\OverBrace[shorten,yshift=3pt]{1-6}{1-6}{\scriptstyle{Q-R}}
\end{NiceArray} \xrightarrow[\substack{\text{match col}\\ \text{with same Q}}] {\texttt{TF}_2 \text{ last token}} \begin{NiceArray}{\left\lgroup c \right\rgroup }

  \vdots\\
  \mathbf{a}_0 \\
  \vdots \\
\end{NiceArray}
$$

First we construct an instructive attention matrix  $\boldsymbol{\alpha}^{(1)} \in \mathbb{R}^{n \times n }$ by set $\mathbf{W}_{QK}^{(1)}$ follow the method in \cref{lem:ins_attn}:

$$
\boldsymbol{\alpha}^{(1)} = \begin{NiceArray}{\left\lgroup cccc \right\rgroup}
\mathbf{A}_{3 \times 3} & \hspace{1cm} & \Block[c]{2-2}<\LARGE>{0} &  \\
& \Ddots^{k \text{ times}} &  &\\
 \Block[c]{2-2}<\LARGE>{0} & & &\\
&  &  &  \mathbf{A}_{2 \times 2}\\
\end{NiceArray} \quad \mathbf{A}_{3 \times 3} = \begin{NiceArray}{\left\lgroup ccc \right\rgroup}
0&1&1\\
1&0&1\\
1&1&0\\
\end{NiceArray}  \quad \mathbf{A}_{2 \times 2} = \begin{NiceArray}{\left\lgroup cc \right\rgroup}
0&1\\
1&0\\
\end{NiceArray}  .
$$

And we set $\mathbf{W}_{V}^{(1)} = \mathbf{I}$, then the first layer only performs the copying operation based on the positions of tokens, and the output of the first layer $\mathbf{H}^{(1)}$ becomes

$$
\mathbf{H}^{(1)} = \mathbf{H}^{(0)} + \mathbf{H}^{(0)} \boldsymbol{\alpha}^{(1)} = \begin{NiceArray}{\left\lgroup ccccccc \right\rgroup}

  \mathbf{q}_0 & \mathbf{q}_0 &  \mathbf{q}_0 & \cdots  &\mathbf{q}_{k-1} &\mathbf{q}_0 & \mathbf{q}_0  \\
  1 & 1 &  1 & \cdots & 1 &  1 & 1 \\
  \mathbf{a}_0 & \mathbf{a}_0 &  \mathbf{a}_0 & \cdots & \mathbf{a}_{k-1} & \mathbf{0} & \mathbf{0} \\
  \mathbf{p}_{0,1,2}  & 
  \mathbf{p}_{0,1,2} &
  \mathbf{p}_{0,1,2} &\cdots & \mathbf{p}_{3k-3,3k-2,3k-1} &
  \mathbf{p}_{3k,3k+1} & \mathbf{p}_{3k,3k+1}  \\
  \mathbf{0} & \mathbf{0} & \mathbf{0} & \cdots & \mathbf{0} & \mathbf{0}  &\mathbf{0} \\
\end{NiceArray},
$$

where $\mathbf{p}_{i,j,\dots} =\mathbf{p}_{i} +  \mathbf{p}_{i} + \dots$ means the sum of several positional encoding. In this way we copy the question, response sign and answer into a vector, making it possible for the match and carry step for the next layer.

Because our question and answer are one-hot vector, so $\mathbf{q}_i^\top \mathbf{q}_j = \begin{cases}
    1 & i = j\\
    0 & \text{otherwise}
\end{cases}$, so we can choose $\mathbf{W}^{(2)}_{QK} \in \mathbb{R}^{d' \times d'}$ as 

$$
\mathbf{W}^{(2)}_{QK} = \begin{NiceArray}{\left\lgroup cc \right\rgroup}
\mathbf{I}_{d_1 \times d_1} & \mathbf{0}_{d_1 \times (d' - d_1)} \\
\mathbf{0}_{(d' - d_1) \times d_1} & \mathbf{0}_{(d' - d_1) \times (d' - d_1)}
\end{NiceArray},
$$

and set $\mathbf{W}_{V}^{(2)} = \mathbf{I}$, then the last output for the 2-second layer $\mathbf{H}^{(2)}[n-1]$

$$
\mathbf{H}^{(2)}[n-1] = \mathbf{H}^{(1)}[n-1] + \mathbf{H}^{(1)} \sigma\left( (\mathbf{H}^{(1)})^\top \mathbf{W}^{(2)}_{QK} \mathbf{H}^{(1)}[n-1] \right) = [6 \mathbf{q}_0; 6; 3 \mathbf{a}_0;3 \mathbf{p}_{3k,3k+1} + 3 \mathbf{p}_{0,1,2} ],
$$

we then set the classification $\mathbf{W}_{O} = [\mathbf{0}_{d_2 \times (d_1 + 1)} , \mathbf{I}_{d_2}, \mathbf{0}_{d_2 \times (n -  (d_1 + d_2 + 1))}]^\top$, in this way, the prediction result of our construct transformer \texttt{TF} is 
$$
\texttt{TF}(\mathbf{H}^{(0)}) = \mathbf{a}_1 .
$$
\end{proof}

\section{Proofs for \cref{sec:tmpl_main}}
\subsection{Proofs for \cref{prop:tmplt_fail}}
\label{sec:prf_tmplt_fail}
In this section, we delve into the reasons why transformers struggle to generalize on different templates, similar to their limitations in reasoning, the primary constraint on the transformer's ability arises from the linear dependence of the predicted results, as mentioned in \cref{prop:out_linear}. Let's recall that a transformer is considered capable of generating output for a specific template if it can accurately classify \textbf{all} sequences generated by that template, so we achieve this by aggregating the prediction results and demonstrating that their sum must be linearly dependent, leading the failure of single-layer of transformer in generalization.


\begin{lemma}
\label{lem:template_com}
For any template $\mathbf{t}$  and \textbf{all possible} sequences set generated by any : $\mathcal{D}_{\texttt{tmpl}}^{(\mathbf{t})} = \{ \mathbf{X}_0,\mathbf{X}_1,\dots,\mathbf{X}_{n^{(\mathbf{t})}-1} \}$, each real world token $\mathbf{x} \in \mathcal{X}$ occur in $\mathcal{D}_{\texttt{tmpl}}^{(\mathbf{t})}$ at each position $\frac{n^{(\mathbf{t})}}{|\mathcal{X}|} $ times.
\end{lemma}
\begin{proof}
The observation in \cref{lem:template_com} is straightforward, as the sequence generated by the same template is token-symmetric: for any sequence $(\mathbf{x}_0,\mathbf{x}_1,\dots,\mathbf{x}_{n-1})$ and any permutation $\pi: \mathcal{X} \rightarrow \mathcal{X}$, $(\mathbf{x}_0,\mathbf{x}_1,\dots,\mathbf{x}_{n-1})$ and $(\pi(\mathbf{x}_0),\pi(\mathbf{x}_1),\dots,\pi(\mathbf{x}_{n-1}))$ belong to the same template. Therefore, for each position, each token should occur in the $\mathcal{D}_{\texttt{tmpl}}^{(\mathbf{t})}$ with the same frequency, resulting in $\frac{n^{(\mathbf{t})}}{|\mathcal{X}|} $ occurrences for each token at any position.
\end{proof}
\textbf{Example}: Let $\mathcal{X} = \{ a,b,c \}$ and two templates of length 2 $\{\alpha \alpha, \alpha \beta \}$. Then in total there are $3$ sequences generated by $\alpha \alpha$: $\{ aa,bb,cc \}$, each token occurs $\frac{n^{(\mathbf{t})}}{|\mathcal{X}|} = \frac{3}{3} = 1$ time at each position. For template $\alpha \beta$, there are $6$ sequences $\{ ab,ac,ba,bc,ca,cb \}$, each token occurs $\frac{6}{3} = 2$ times at each position.

\begin{lemma}
\label{lem:template_seq_com}
Given any two different templates $\mathbf{t},\mathbf{t}'$, and \textbf{all possible} sequences generated $\mathcal{D}_{\texttt{tmpl}}^{(\mathbf{t})} = \{ \mathbf{X}_0,\dots,\mathbf{X}_{n^{(\mathbf{t})}-1} \}$, $\mathcal{D}_{\texttt{tmpl}}^{(\mathbf{t}')} = \{ \mathbf{X}'_0,\dots,\mathbf{X}'_{n^{(\mathbf{t}')}-1} \}$, there exist $\lambda,\lambda' \in \mathbb{Z}_{+}$ such that the combined occurrence of each token at each position in $\mathcal{D}_{\texttt{tmpl}}^{(\mathbf{t})}$ is equal to the combined occurrence of each token at each position in $\mathcal{D}_{\texttt{tmpl}}^{(\mathbf{t}')}$:
        \begin{equation}\label{eq:template_seq_com_1}
            \sqcap_{i \in [n^{(\mathbf{t})}]} (\lambda \otimes \mathbf{X}_{i}) = \sqcap_{j \in [n^{(\mathbf{t'})}]} (\lambda' \otimes \mathbf{X}_{j}),
        \end{equation}

    where $\sqcap$ denotes the token combination operation defined in \cref{def:data_dependence}, which computes the occurrence of each token at each position.
\end{lemma}
\begin{proof}
Based on the observation in \cref{lem:template_com}, \cref{lem:template_seq_com} is straightforward. Since each token occurs $\frac{n^{(\mathbf{t})}}{|\mathcal{X}|} $ times at each position, we can let $\lambda  = n^{(\mathbf{t}')}$ and $\lambda' =  n^{(\mathbf{t})}$. Then, for the left part $\sqcap_{i \in [n^{(\mathbf{t})}]} (\lambda \otimes \mathbf{X}_{i})$, each token occurs $n^{(\mathbf{t'})}  \frac{n^{(\mathbf{t})}}{|\mathcal{X}|}$ times at each position. Similarly, for the right part $\sqcap_{j \in [n^{(\mathbf{t'})}]} (\lambda \otimes \mathbf{X}_{j})$, each token also occurs $n^{(\mathbf{t})}  \frac{n^{(\mathbf{t'})}}{|\mathcal{X}|} = n^{(\mathbf{t'})}  \frac{n^{(\mathbf{t})}}{|\mathcal{X}|}$ times. This completes the proof.
\end{proof}

\textbf{Example}: Let $\mathcal{X} = \{ a,b,c \}$ and two templates of length 2 $\{\alpha \alpha, \alpha \beta \}$. Then in total there are $n^{\alpha \alpha}=3$ sequences generated by $\alpha \alpha$: $\{ aa,bb,cc \}$, each token occurs $1$ time at each position. For template $\alpha \beta$, there are $n^{\alpha \beta}=6$ sequences $\{ ab,ac,ba,bc,ca,cb \}$, each token occurs $2$ times at each position. We can find positive integers $\lambda_1 = 6$ and $\lambda_2 = 3$ such that $3(ab \oplus ac \oplus ba \oplus bc \oplus ca \oplus cb) = 6 ([a,b,c],[a,b,c],[a,b,c]) = 6 (aa \oplus bb \oplus cc)$.  Note that in Equation \ref{eq:template_seq_com_1}, both sides have $n^{(\mathbf{t})} n^{(\mathbf{t'})}$ tokens. If we add a response sign at the end of each sequence, i.e., $\{ aa,bb,cc \} \xrightarrow{\text{add sign } r} \{ aar,bbr,ccr \}$, Lemma \ref{lem:template_seq_com} still holds.

Since the last token of the sequence input into the transformer is the same response sign, based on \cref{lem:template_seq_com} and \cref{def:data_dependence}, for any two templates, the sequences generated by them are dependent. This dependence leads to the following lemma:

\begin{lemma}
\label{lem:template_pre_com}
Given any two different templates $\mathbf{t},\mathbf{t}'$, and \textbf{all possible} sequences generated $\mathcal{D}_{\texttt{tmpl}}^{(\mathbf{t})} = \{ \mathbf{X}_0,\dots,\mathbf{X}_{n^{(\mathbf{t})}-1} \}$, $\mathcal{D}_{\texttt{tmpl}}^{(\mathbf{t}')} = \{ \mathbf{X}'_0,\dots,\mathbf{X}'_{n^{(\mathbf{t}')}-1} \}$. For any single-layer transformer model $\texttt{TF}$, let $\mathbf{o}^{(i)} = \texttt{TF}(\mathbf{X}_i)$ and $\mathbf{o}^{'(i)} = \texttt{TF}(\mathbf{X}'_i)$ denote the model predictions for these two templates. Then we have:

\begin{equation}
\lambda_1 \sum_{i = 0}^{n^{(\mathbf{t}')}-1} \mathbf{p}^{(i)} = \lambda_2 \sum_{i = 0}^{n^{(\mathbf{t}')}-1} \mathbf{p}^{'(i)}.
\end{equation}
\end{lemma}

\begin{proof}
As the sequences generated by any two templates are dependent, we can apply \cref{prop:out_linear} to directly complete the proof.
\end{proof}

Note that any sequence generated by the same template should be classified by the same template. Therefore, if the model can generalize on both templates, the sum of the prediction output for template $\mathbf{t}$, $\sum_{i = 0}^{n^{(\mathbf{t})}-1}\mathbf{p}^{(i)}$, should have the maximum value at position $y_1$, and $\sum_{i = 0}^{n^{(\mathbf{t'})}-1} \mathbf{p}^{'(i)} $ should have the maximum value at position $y_2$, where $y_1 \neq y_2$ since different templates belong to different classes. This contrasts with \cref{lem:template_pre_com}, and thus, we can conclude that transformers cannot generalize on any two different templates.
\subsection{Proof for \cref{prop:tmplt_succ}}
\label{sec:cons_tmplt}
In this section, we will construct 2-layer transformer that can generalize on the template task, our construct transformer can first parse the sequence into the template, and then use the memorization ability of one-layer transformer, mapping the template to the corresponding label, follow such parsing-mapping procedure, our constructed transformer can generalize on the template task.

Recall the data construction process, we consider an input sequence generated by a template length $k$, a $d' \times n$ matrix $\mathbf{H}^{(0)}$ where $n = k + 1$:
\\
\\
\\
$$
\mathbf{H}^{(0)} = 
\begin{NiceArray}{\left\lgroup ccccc \right\rgroup l}
  \mathbf{x}_0 & \mathbf{x}_1 &  \cdots & \mathbf{x}_{k-1} & \mathbf{0} & \rightarrow d \times n \\
  0 & 0 &  \cdots & 0 & 1 &  \rightarrow 1 \times n \\
  \mathbf{p}_0 & \mathbf{p}_1 & \cdots & \mathbf{p}_{k-1} & \mathbf{p}_{k} & \rightarrow n \times n \\
  \mathbf{0} & \mathbf{0} & \cdots & \mathbf{0} & \mathbf{0}  & \rightarrow (d' - (n + 1 + d) ) \times n \\
\CodeAfter
\OverBrace[shorten,yshift=3pt]{1-1}{1-4}{\scriptstyle{\substack{\text{sequence} \\ \text{by template}}}}
\OverBrace[shorten,yshift=3pt]{1-5}{1-5}{\scriptstyle{\substack{\text{response }\\ \text{sign}}}}
\end{NiceArray}
$$
We illustrate our construction by using a template length 3 $\alpha \beta \beta$,which generates the sequence $(\mathbf{a},\mathbf{b},\mathbf{b})$: 
\\

\begin{equation*}
\begin{split}
\mathbf{H}^{(0)} = 
&\begin{NiceArray}{\left\lgroup cccc \right\rgroup }
  \mathbf{a} & \mathbf{b} &  \mathbf{b} & \mathbf{0} \\
  0 & 0 &  0  & 1 \\
  \mathbf{p}_0 & \mathbf{p}_1 & \mathbf{p}_2 & \mathbf{p}_{3} \\
  \mathbf{0} & \mathbf{0} & \mathbf{0} & \mathbf{0} \\
\CodeAfter
\OverBrace[shorten,yshift=3pt]{1-1}{1-3}{\scriptstyle{\substack{\text{from template} \\ \alpha \beta \beta}}}
\OverBrace[shorten,yshift=3pt]{1-4}{1-4}{\scriptstyle{\substack{\text{response }\\ \text{sign}}}}
\end{NiceArray}\\ \xrightarrow[\substack{\text{check position}\\ \text{with the same token} }]{\texttt{TF}_1} &\begin{NiceArray}{\left\lgroup cccc \right\rgroup }
  \mathbf{a} & \mathbf{b} &  \mathbf{b} & \mathbf{0} \\
  0 & 0 &  0  & 1 \\
  \mathbf{p}_0 + \mathbf{p}_0& \mathbf{p}_1 + \mathbf{p}_{1,2} & \mathbf{p}_2 + \mathbf{p}_{1,2} & \mathbf{p}_{3}+  \mathbf{p}_{3} \\
  \mathbf{0} & \mathbf{0} & \mathbf{0} & \mathbf{0} \\
\CodeAfter
\UnderBrace[shorten,yshift=3pt]{4-1}{4-3}{\scriptstyle{\substack{\text{positional encoding} \\ \text{as template representation}}}}
\end{NiceArray}\\ \\ \\ \xrightarrow[\substack{\text{map template}\\ \text{to vectors} }]{\texttt{TF}_2} &\begin{NiceArray}{\left\lgroup c \right\rgroup l}
  \vdots & \\
  (200)_{(3)} & \Block{3-1}{\hspace*{2mm} \substack{\text{encode each} \\\text{row to a}\\ \text{ternary number}}}\\
  (021)_{(3)} & \\
  (012)_{(3)} & \\
  \vdots \\
  \CodeAfter
  \SubMatrix{\}}{2-1}{4-1}{.}[xshift=-17mm]
\end{NiceArray} 
\end{split}
\end{equation*}
\\

First we set  our transformer use attention mechanism parse each sequence into template, recall the match layer in  \cref{sec:cons_icl}, each token only focus on the position of the same token, in this way, the position dimension can be considered as a dimension of template sequence:

$$
\mathbf{W}^{(1)}_{QK} = \begin{NiceArray}{\left\lgroup cc \right\rgroup}
\mathbf{I}_{d \times d} & \mathbf{0}_{d \times (d' - d)} \\
\mathbf{0}_{(d' - d) \times d} & \mathbf{0}_{(d' - d) \times (d' - d)}
\end{NiceArray},
$$

$$
\mathbf{W}^{(1)}_{V} = \begin{NiceArray}{\left\lgroup cc \right\rgroup}
-\mathbf{I}_{(d+1) \times (d+1)} & \mathbf{0}_{(d+1) \times (d' - (d + 1))} \\
\mathbf{0}_{(d' - (d+1)) \times (d+1)} & \mathbf{0}_{(d' - (d+1)) \times (d' - (d + 1))}
\end{NiceArray} + \begin{NiceArray}{\left\lgroup cc \right\rgroup}
\mathbf{0}_{(d+1) \times (n)} & \mathbf{0}_{(d+1) \times (d' - n)} \\
\mathbf{I}_{n \times n }  & \mathbf{0}_{n \times (d' - n)}\\
\mathbf{0}_{(d' - (n + 1 + d))  \times (n)} & \mathbf{0}_{(d' - (n + 1 + d)) \times (d' - n)} \\
\end{NiceArray}.
$$

Then the resulting sequence to be input to the second layer is
\begin{equation*}
\begin{split}
\mathbf{H}^{(1)} &= \mathbf{H}^{(0)} + \mathbf{W}^{(1)}_{V}\mathbf{H}^{(0)} \boldsymbol{\alpha}^{(1)} \\&=  \mathbf{H}^{(0)}  + \begin{NiceArray}{\left\lgroup ccccc \right\rgroup}
  -\mathbf{x}_0 & -\mathbf{x}_1 &  \cdots & -\mathbf{x}_{k-1} & \mathbf{0} \\
  0 & 0 &  \cdots & 0 & -1 \\
  \mathbf{p}'_0 & \mathbf{p}'_1 & \cdots & \mathbf{p}'_{k-1} & \mathbf{p}_{k } \\
  \mathbf{0} & \mathbf{0} & \cdots & \mathbf{0} & \mathbf{0} \\
\end{NiceArray}\\&
=   \begin{NiceArray}{\left\lgroup ccccc \right\rgroup }
  \mathbf{0} & \mathbf{0} &  \cdots & \mathbf{0} & \mathbf{0} \\
  0 & 0 &  \cdots & 0 & 0 \\
  \mathbf{p}'_0 + \mathbf{p}_0 & \mathbf{p}'_1 + \mathbf{p}_1 & \cdots & \mathbf{p}'_{k-1} + \mathbf{p}_{k-1} & \mathbf{p}_{k } + \mathbf{p}_{k} \\
  \mathbf{0} & \mathbf{0} & \cdots & \mathbf{0} & \mathbf{0} \\
\end{NiceArray}.
\end{split}
\end{equation*}

The second part can be seen as a sequence of template, take $k = 4$ and template be $\alpha \beta \alpha \gamma$ as example, then $\mathbf{p}'_0 = \mathbf{p}'_2 = \mathbf{p}_0 + \mathbf{p}_2, \mathbf{p}'_1 = \mathbf{p}_1, \mathbf{p}'_3 = \mathbf{p}_3$, then the result template representation is $\begin{NiceArray}{\left\lgroup cccc \right\rgroup }
  2&0&1&0\\
  0&2&0&0\\
  1&0&2&0\\
  0&0&0&2
\end{NiceArray}$. Note that $\mathbf{p}'_i = \mathbf{p}'_j$ if and only if $\mathbf{x}_i = \mathbf{x}_j$. We can further prove that the output sequence parsed by the first layer of the transformer satisfies the following:
\begin{itemize}
    \item Sequence from the same template have the same representation.
    \item Different templates have different, no-parallel representation matrix.
\end{itemize}

The first point is straightforward since this representation is based on positional encoding and is token-invariant. For the second point, since the elements in this representation only consist of the values ${0,1,2}$, and the diagonal value is always $2$ for any two representations $\mathbf{H}^{(1)}$ and $\mathbf{H}^{'(1)}$, if there exists a $\lambda$ such that $\mathbf{H}^{(1)} = \lambda \mathbf{H}^{'(1)}$, then we must have $\lambda = 1$ and $\mathbf{H}^{(1)} = \mathbf{H}^{'(1)}$. Therefore, each representation is non-parallel.

The role of second layer, incorporate with the classification layer, is to map the sequence to its corresponding label, as we define different template have different label, and based on the two property of the output of the first layer, we can utilize the memorization ability of the first layer transformer, as demonstrated in  \cref{sec:prof_mem}, and following the construction procedure, construct a transformer with $n$ attention heads that can correctly map each sequence to corresponding label.In addition to that, since template encoding is a more structured form of data compared to randomly assembled sequences, we provide one possible weight that can map the templates to non-parallel vectors using only a single attention head:

First we construct a inductive head with $\boldsymbol{\alpha} = \begin{NiceArray}{\left\lgroup ccccc \right\rgroup }
  \mathbf{0} & \mathbf{0} &  \cdots & \mathbf{0} & \mathbf{0} \\
  3^{k-1} & 3^{k-2} & \cdots & 3^0 & 0 \\
\end{NiceArray}^\top$, and set $\mathbf{W}_{V}^{(2)} = \mathbf{I}$, in this way the last token of the sencond layer is $\sum_{i=1}^{k} 3^{k - i} (\mathbf{p}_i +  \mathbf{p}'_i) $, as each element of ${\mathbf{p}_i +  \mathbf{p}'_i}$ only consists of three integers ${0,1,2}$, and $(\mathbf{p}_i +  \mathbf{p}'_i)[i] = 2$, in this way, for each row of $\mathbf{H}^{(1)}$, we can encode it into a unique ternary number, for example, given a vector $[2,1,1,0,0]$,  we can encode it to $2 \times 3^{4} + 1 \times 3^{3} + 1 \times 3^{2} + 0 \times 3^{1} + 0 \times 3^{0} = 21100_{(3)}$, and the representation for $\alpha \beta \alpha \gamma$, we encode it row by row: $\begin{NiceArray}{\left\lgroup cccc \right\rgroup }
  2&0&1&0\\
  0&2&0&0\\
  1&0&2&0\\
  0&0&0&2
\end{NiceArray} \rightarrow \begin{NiceArray}{\left\lgroup c \right\rgroup }
  (2010)_{(3)}\\
  (0200)_{(3)}\\
  (1020)_{(3)}\\
  (0002)_{(3)}
\end{NiceArray}$. Now consider the representation at position $k$, as $(\mathbf{p}_k + \mathbf{p}'_k)[k-1] = 2$ for all representation, the resulting vector at position $k$ should be $(\dots 2)_{(3)}$, in this way we can ensure different template will project into no-parallel vectors, as if two template representaion vector is parallel, they must have the \textbf{same} representation vector, which further indicate they must have the same template representation at each row.

Then utilizing \cref{lem:vector_mapping}, we can construct a $\mathbf{W}_C \in \mathbb{R}^{ d' \times C}$ map each vector to their distinct label. In this way, a part from the $n$ attention head transformer, we further proof there exists a 2-layer transformer with single head that can accomplish our generalization task.


\section{Construction for \cref{sec:cont_gen}}
\label{sec:cons_cont_gen}
In this section, we provide a construction for a 3-layer transformer that is capable of handling contextual generalization task. Our constructed transformer follows a parsing-copy-match process, the transformer parses each question into corresponding template, this is achieved by utilizing the parsing process in the template generalization task, we use the constrained attention to force the model focuses only on tokens that belong to the same question block. Next, in the second layer, we mix the question and corresponding answer together by utilizing inductive attention. Then the final layer retrieve the corresponding question-answer representation and transform it into the final answer.

Let us consider an input sequence with k question-answer pairs, each question being of sequence length $l$. We represent this input sequence as $\mathbf{H}^{(0)} \in \mathbb{R}^{d' \times n}$, where $n = (l+2)k+l+1$.
\\
\\
$$
\mathbf{H}^{(0)} = 
\begin{NiceArray}{\left\lgroup cccccccccc \right\rgroup l}

  \mathbf{x}_0^{(0)} & \cdots & \mathbf{x}_{l-1}^{(0)} & \mathbf{0} &  \mathbf{0} & \cdots & \mathbf{x}_0^{'(r)} & \cdots & \mathbf{x}_{l-1}^{'(r)} & \mathbf{0} & \rightarrow d_1 \times n \\
  0 & \cdots & 0 & 1 &  0 & \cdots & 0 & \cdots & 0 & 1 &  \rightarrow 1 \times n \\
  \mathbf{0} & \cdots & \mathbf{0} & \mathbf{0} & \mathbf{a}_0 & \cdots & \mathbf{0} & \cdots & \mathbf{0} & \mathbf{0} & \rightarrow d_2 \times n \\
   \mathbf{p}_0 & \cdots & \mathbf{p}_{l-1} & \mathbf{p}_{l} & \mathbf{p}_{l+1} &\cdots & \mathbf{p}_{(l+2)k} & \cdots & \mathbf{p}_{ (l+2)k + l-1} & \mathbf{p}_{(l+2)k + l} &\rightarrow n \times n \\
  \mathbf{0} & \cdots & \mathbf{0} & \mathbf{0} &  \mathbf{0} & \cdots & \mathbf{0} & \cdots & \mathbf{0} & \mathbf{0} & \rightarrow (d' - n - d) \times n \\
\CodeAfter
\OverBrace[shorten,yshift=3pt]{1-1}{1-3}{\scriptstyle{\text{Question}}}
\OverBrace[shorten,yshift=3pt]{1-4}{1-4}{\scriptstyle{\substack{\text{Response} \\ \text{Sign}}}}
\OverBrace[shorten,yshift=3pt]{1-5}{1-5}{\scriptstyle{\substack{\text{Answer}}}}
\UnderBrace[shorten,yshift=3pt]{5-1}{5-6}{\scriptstyle{\substack{(Q-R-A) \times k}}}
\end{NiceArray}
$$
\\
\\
Where $r \in [k]$ is a random choose question index, we can set $r = 1$ , $k=2$ and $l = 2$, which means there are only two templates $\alpha \alpha$ and $\alpha \beta$ ,  to illustrate how our construction works:
\\
\\
\\
\begin{equation*}
\begin{split}
\mathbf{H}^{(0)} =& 
\begin{NiceArray}{\left\lgroup ccccccccccc \right\rgroup}
  \mathbf{a} & \mathbf{a} & \mathbf{0} &  \mathbf{0} & \mathbf{a} & \mathbf{b} & \mathbf{0} &  \mathbf{0} & \mathbf{b} & \mathbf{b} & \mathbf{0} \\
  0 & 0 & 1 & 0 &  0 & 0 & 1 & 0 & 0 & 0 & 1  \\
  \mathbf{0} & \mathbf{0} & \mathbf{0} & \mathbf{a}_0 & \mathbf{0} & \mathbf{0} & \mathbf{0} & \mathbf{a}_1 & \mathbf{0} & \mathbf{0} & \mathbf{0} \\
   \mathbf{p}_0 & \mathbf{p}_1 & \mathbf{p}_2 & \mathbf{p}_3 & \mathbf{p}_4 & \mathbf{p}_5 & \mathbf{p}_6 & \mathbf{p}_7 & \mathbf{p}_8 & \mathbf{p}_{9} & \mathbf{p}_{10}\\
  \mathbf{0} & \mathbf{0} & \mathbf{0} &\mathbf{0} & \mathbf{0} & \mathbf{0} &\mathbf{0} & \mathbf{0} & \mathbf{0} &\mathbf{0} & \mathbf{0} \\
\CodeAfter
\OverBrace[shorten,yshift=3pt]{1-1}{1-2}{\scriptstyle{\text{Question}}}
\OverBrace[shorten,yshift=3pt]{1-3}{1-3}{\scriptstyle{\substack{\text{Response} \\ \text{Sign}}}}
\OverBrace[shorten,yshift=3pt]{1-4}{1-4}{\scriptstyle{\substack{\text{Answer}}}}
\end{NiceArray} \\
\xrightarrow[\substack{\text{check token}\\ \text{among Q-A block \&} \\ \text{align representation}}]{\texttt{TF}_1}& \begin{NiceArray}{\left\lgroup ccccccccccc \right\rgroup}
  \vdots&\vdots& \vdots& \vdots& \vdots&  \vdots& \vdots& \vdots& \vdots& \vdots &\vdots\\
  \mathbf{0} & \mathbf{0} & \mathbf{0} & \mathbf{a}_0 & \mathbf{0} & \mathbf{0} & \mathbf{0} & \mathbf{a}_1 & \mathbf{0} & \mathbf{0} & \mathbf{0} \\
  \vdots&\vdots& \vdots& \vdots& \vdots&  \vdots& \vdots& \vdots& \vdots& \vdots &\vdots\\
   2 & 1 & 0 & 0 & 2 & 0 & 0 & 0 & 2 & 1 & 0\\
  1 & 2 & 0 &0 & 0 & 2 &0 & 0& 1 & 2 & 0 \\
  \vdots&\vdots& \vdots& \vdots& \vdots&  \vdots& \vdots& \vdots& \vdots& \vdots &\vdots\\
  \CodeAfter
  \UnderBrace[shorten,yshift=3pt]{6-1}{6-2}{\scriptstyle{\substack{\text{template} \\ \alpha \alpha}}}
  \UnderBrace[shorten,yshift=3pt]{6-5}{6-6}{\scriptstyle{\substack{\text{template} \\ \alpha \beta}}}
  \UnderBrace[shorten,yshift=3pt]{6-9}{6-10}{\scriptstyle{\substack{\text{template} \\ \alpha \alpha}}}
\end{NiceArray} \\ \\ 
\xrightarrow[\substack{\text{copy answer}\\ \text{compare each template}}]{\texttt{TF}_2}& \begin{NiceArray}{\left\lgroup ccccccccccc \right\rgroup}
  \vdots&\vdots& \vdots& \vdots& \vdots&  \vdots& \vdots& \vdots& \vdots& \vdots &\vdots\\
  \mathbf{a}_0 & \mathbf{a}_0 & \mathbf{0} & \mathbf{a}_0 & \mathbf{a}_1  & \mathbf{a}_1 & \mathbf{0} & \mathbf{a}_1 & \mathbf{0} & \mathbf{0} & \mathbf{a}_0 + \mathbf{a}_1  \\
  \vdots&\vdots& \vdots& \vdots& \vdots&  \vdots& \vdots& \vdots& \vdots& \vdots &\vdots\\
   0 & 0 & 0 & 0 & 0 & -2 & 0 & 0 & 2 & 1 & 1\\
  0 & 0 & 0 &0 & -2 & 0 &0 & 0& 1 & 2 & 1 \\
  \vdots&\vdots& \vdots& \vdots& \vdots&  \vdots& \vdots& \vdots& \vdots& \vdots &\vdots\\
  \CodeAfter
  \UnderBrace[shorten,yshift=3pt]{6-1}{6-2}{\scriptstyle{\substack{\text{same template} \\ \text{template row all zero}}}}
  \UnderBrace[shorten,yshift=3pt]{6-5}{6-6}{\scriptstyle{\substack{\text{different template} \\ \text{ exist no-zero elements}}}}
  \UnderBrace[shorten,yshift=3pt]{6-9}{6-10}{\scriptstyle{\substack{\text{template} \\ \alpha \alpha}}}
\end{NiceArray} \\ \\ \\
\xrightarrow[\substack{\text{focus on col that template}\\ \text{exists no-zero elements}}]{\texttt{TF}_3, \text{ last token}} &
\begin{NiceArray}{\left\lgroup c \right\rgroup l}
  \vdots\\
  \mathbf{a}_0 + \mathbf{a}_1 - \gamma \mathbf{a}_1 & (\gamma \geq 1)\\
  \vdots
\end{NiceArray}
\end{split}
\end{equation*}
\\
\\
In the first layer, we parse each question into a template. To achieve this, we set the weight $\mathbf{W}_{QK}^{(1)}$ as follows:

$$
\mathbf{W}_{QK}^{(1)} = \begin{NiceArray}{\left\lgroup cc \right\rgroup}
\mathbf{I}_{d_1 \times d_1} & \mathbf{0}_{d_1 \times (d' - d_1)} \\
\mathbf{I}_{(d' - d_1) \times d_1} & \mathbf{0}_{(d' - d_1) \times (d' - d_1)} \\
\end{NiceArray} + \begin{NiceArray}{\left\lgroup ccc \right\rgroup}
\mathbf{0}_{d \times d} & \mathbf{0}_{d \times n} & \mathbf{0}_{d \times (d' - (n + d))}\\
\mathbf{0}_{n \times d} & \mathbf{W}_{pp} & \mathbf{0}_{n \times (d' - (n + d))}\\
\mathbf{0}_{(d' - (n + d)) \times d} & \mathbf{0}_{(d' - (n + d))  \times n} & \mathbf{0}_{(d' - (n + d)) \times (d' - (n + d))}\\
\end{NiceArray}.
$$

Here, $d = d_1 + d_2 + 1$, and $\mathbf{W}_{pp}$ is an $n \times n$ matrix. As discussed in \cref{lem:ins_attn}, we can set:

$$
\mathbf{W}_{pp} = \begin{NiceArray}{\left\lgroup cccc \right\rgroup}
\mathbf{0}_{(l + 2) \times (l + 2)} & \hspace{1cm} & \Block[c]{2-2}<\LARGE>{-\infty} &  \\
& \Ddots^{k \text{ times}} &  &\\
 \Block[c]{2-2}<\LARGE>{-\infty} & & &\\
&  &  &  \mathbf{0}_{(l + 1) \times (l + 1)}\\
\end{NiceArray}.
$$

This way, the attention is constrained to each question-answer block, making it easier to parse each question into a template. We then set

$$
\mathbf{W}^{(1)}_{V} = 
\begin{NiceArray}{\left\lgroup cccc \right\rgroup}
\mathbf{0}_{(d' - l-2) \times d }&\mathbf{0}_{(d' - l-2) \times (l+2)}  & \dots & \mathbf{0}_{(d' - l-2)\times (l+1)}\\
\mathbf{0}_{(l+2) \times d }&\mathbf{I}_{(l+2) \times (l+2)}  & \dots & \mathbf{I}_{(l+2) \times (l+1)}\\
\CodeAfter
\UnderBrace[shorten,yshift=3pt]{2-2}{2-3}{\scriptstyle{k \text{ times}}}
\end{NiceArray}.
$$
\\

Here, we assume that the hidden size $d'$ is large enough, so that $d' - (d + n) > l + 2$. This enables the model to save the representation disentangled. As a result, the output of the first layer has the same positional encoding as the input sequence.

In the second layer, we perform a pre-match and copy procedure by constructing an instructive attention $\boldsymbol{\alpha}^{(2)} \in \mathbb{R}^{n \times n}$.

$$
\boldsymbol{\alpha}^{(2)} = \begin{NiceArray}{\left\lgroup cccc \right\rgroup}
\mathbf{A}_{(l + 2) \times (l + 2)} & \Hspace*{2cm} & \Block[c]{2-2}<\LARGE>{0} &  \\
& \Ddots^{(k-1) \text{ times}} &  & \\
 \Block[c]{2-2}<\LARGE>{0} & & & \\
&  &  &  \mathbf{A}_{(l + 1) \times (l + 1)}\\
\end{NiceArray}^\top + \begin{NiceArray}{\left\lgroup cc \right\rgroup l}
\mathbf{0}_{(l + 2) \times (d - (l+1))} & \mathbf{B}_{(l + 2) \times (l + 1)} & \Block{3-1}{\hspace*{2mm} \substack{k \text{ times}}} \\
\vdots & \vdots \\
\mathbf{0}_{(l + 1) \times (d - (l+1))} & \mathbf{0}_{(l + 1) \times (l + 1)} \\
\CodeAfter
\SubMatrix{\}}{2-1}{2-2}{.}[xshift=-51mm]
\end{NiceArray}^\top.
$$

Here, $\mathbf{A}_{(l + 2) \times (l + 2)} = \begin{NiceArray}{\left\lgroup cc \right\rgroup}
\mathbf{0}_{(l+2) \times (l+1)} & \mathbf{1}_{(l+2) \times 1}
\end{NiceArray}$, $\mathbf{A}_{(l + 1) \times (l + 1)} = \begin{NiceArray}{\left\lgroup c \right\rgroup}
\mathbf{0}_{l \times (l+1)} \\
\mathbf{1}_{1 \times (l+1)}
\end{NiceArray}$, and $\mathbf{B}_{(l + 2) \times (l + 1)} = \begin{NiceArray}{\left\lgroup c \right\rgroup}
\mathbf{I}_{(l+1) \times (l+1)} \\
\mathbf{0}_{1 \times (l+1)}
\end{NiceArray}$.

The left part of $\boldsymbol{\alpha}^{(2)}$ copies the answer to each column of the question (template), while the right part of the attention tries to compare each row of the template with the last template. As we set $\mathbf{W}_{V}^{(2)} = \begin{NiceArray}{\left\lgroup cc \right\rgroup}
\mathbf{I}_{d \times d} & \mathbf{0}_{d \times (d' -d)}\\
\mathbf{0}_{(d' -d) \times d} & -\mathbf{I}_{(d' -d) \times (d' -d)}\\
\end{NiceArray}$, 
if there is a template that is the same as the last one, then their template representation is zero. In this way, we perform a copy procedure, accompanied by the template pre-matching procedure in the second layer.

In the third layer, we utilize $2 l$ attention heads. If any row of the template representation is not zero, there exist at least one attention head that subtracts the corresponding answer from the last token. To achieve this, we set the weight matrices as follows:
$$
\mathbf{W}^{(3)}_{QK_i} =  \begin{NiceArray}{ \left\lgroup cc \right\rgroup}
\mathbf{0}_{(d' - l) \times (d' - l)}& \mathbf{0}_{(d' - l) \times l)} \\
\mathbf{0}_{l \times (d' - l)}& \mathbf{M}_{i} \\
\end{NiceArray} \quad \mathbf{W}^{(3)}_{V_i} = -\mathbf{I}.
$$

Here, $\mathbf{M}_{i} \in \mathbb{R}^{l \times l}$. For the first $l$ heads ($0 \leq i < l$), only the element at position $(i,i)$ is $1$ and the rest are zero. For $l \leq i < 2l$, only the element at position $(i-l,i-l)$ is $-1$ and the rest are zero. Together with $\mathbf{W}^{(3)}_{V_i} = -\mathbf{I}$, the third layer will focus on the answer that has non-zero elements and subtract the corresponding answer from the final representation. 

The output for the last token is as follows:

$$
\texttt{TF}_3(\mathbf{H}^{(3)})[n-1] = \begin{NiceArray}{ \left\lgroup c \right\rgroup l}
\vdots \\
\sum_{i=0}^{k-1} \mathbf{a}_i -  \sum_{i=0, i\neq r}^{k-1} \gamma_i \mathbf{a}_i & (\gamma_i \geq 1)\\
\vdots
\end{NiceArray}.
$$

As we assume the tokens are one-hot, we can set the classifier layer as 

$$
\mathbf{W}_{O} = \begin{NiceArray}{ \left\lgroup ccc \right\rgroup}
\mathbf{0}_{(d_1  + 1) \times d_2 } & \mathbf{I}_{d_2  \times d_2} & \mathbf{0}_{(d' - d  \times d_2)}\\
\mathbf{0}_{(d_1  + 1) \times (c - d_2 )} & \mathbf{0}_{d_2  \times (c - d_2 )} & \mathbf{0}_{(d' - d  \times (c - d_2 ))}\\
\end{NiceArray}.
$$

By letting $c = d_2$, the final prediction result is $\sum_{i=0}^{k-1} \mathbf{a}_i -  \sum_{i=0, i\neq r}^{k-1} \gamma_i \mathbf{a}_i \, (\gamma_i \geq 1)$, which has the largest element corresponding to the final answer $\mathbf{a}_r$. 


\section{Expanding from ReLU to Softmax Attention}
\label{sec:soft_max}
When considering a transformer with a single layer and a single head, the results for a single-layer ReLU attention transformer still hold.

In the case of softmax attention, which replaces the activation function in Equation \ref{eq:transformer_attn} from ReLU ($\sigma(x) = \max\{0,x\}$) to softmax ($\text{softmax}(\mathbf{x})_i = \frac{e^{x_i}}{\sum_{j=1}^{K}e^{x_j}}$), the input sequence $\mathbf{X} \in \mathbb{R}^{d' \times n}$. For a single attention head with parameters $\mathbf{W}_{V}$ and $\mathbf{W}_{QK}$, we define $\sigma'(\mathbf{x}_j,\mathbf{x}_k) = \exp(\mathbf{x}_j^\top  \mathbf{W}_{QK} \mathbf{x}_k)$. Equation \ref{eq:1_lsa_k} can be rewritten for softmax attention as follows:

\begin{align}\label{eq:transformer_spft_attn_1}
    \texttt{TF}_{\text{soft}}(\mathbf{H})[k] &= \mathbf{x}_{n-1}  + \sum_{i=0}^{n-1}  \frac{ \mathbf{W}_{V}\mathbf{x}_i \sigma'(\mathbf{x}_i,\mathbf{x}_{k}) }{\sum_{j=0}^{n-1}\sigma'(\mathbf{x}_j,\mathbf{x}_{k})}
\end{align}

\begin{proposition}
\label{prop:out_linear_softmax}
If the input sequences $\mathbf{X}^{(0)},\mathbf{X}^{(1)},\dots, \mathbf{X}^{(N-1)} \in \mathcal{X}^n$  are \textbf{dependent}, then for any single layer single head transformer with softmax attention $\texttt{TF}_{\text{soft}}$, we have

 $$
 \lambda'_0 (\texttt{TF}_{\text{soft}}(\mathbf{X}^{(0)})[n-1])  + \lambda'_1 (\texttt{TF}_{\text{soft}}(\mathbf{X}^{(1)})[n-1]) + \dots + \lambda'_{N-1} (\texttt{TF}_{\text{soft}}(\mathbf{X}^{(N-1)})[n-1]) = \mathbf{0}.
 $$

where $\lambda'_i = \sum_{j=0}^{n-1}\sigma'(\mathbf{x}^{(i)}_j,\mathbf{x}_{n-1}) \lambda_i$, $\{\lambda_i\}_{i=0}^{N-1}$ represents the coefficients defined in \cref{def:data_dependence}. 

\end{proposition}

\begin{proof}
Consider the last token 

$$
\texttt{TF}_{\text{soft}}(\mathbf{X}^{(e)})[n-1] = \mathbf{x}_{n-1}  + \sum_{i=0}^{n-1}  \frac{ \mathbf{W}_{V}\mathbf{x}^{(e)}_i \sigma'(\mathbf{x}^{(e)}_i,\mathbf{x}_{n-1}) }{\sum_{j=0}^{n-1}\sigma'(\mathbf{x}^{(e)}_j,\mathbf{x}_{n-1})}
$$

then we have 

\begin{equation}
\begin{split}
    &\lambda'_0 (\texttt{TF}_{\text{soft}}(\mathbf{X}^{(0)})[n-1])  + \lambda'_1 (\texttt{TF}_{\text{soft}}(\mathbf{X}^{(1)})[n-1]) + \dots + \lambda'_{N-1} (\texttt{TF}_{\text{soft}}(\mathbf{X}^{(N-1)})[n-1])\\
    =& \mathbf{x}_{n-1} (\sum_{i=0}^{n-1} \sum_{j=0}^{n-1}\lambda_i (\sigma'(\mathbf{x}^{(i)}_j,\mathbf{x}_{n-1})) ) +  (\sum_{i=0}^{n-1} \sum_{j=0}^{n-1} (\lambda_i \mathbf{W}_{V}\mathbf{x}^{(i)}_j \sigma'(\mathbf{x}^{(i)}_j,\mathbf{x}_{n-1}) )
\end{split}
\end{equation}

Similar with the proof techique in \cref{prop:out_linear} we can divide the sequences into two groups $\mathcal{I}_{+}$ and $\mathcal{I}_{-}$ so in each position $j$, both side have the same occurrence for each token:

\begin{equation}
\sqcap_{i \in \mathcal{I}_{+}} (\lambda_i \otimes \mathbf{x}^{(i)}_j) = \sqcap_{i \in \mathcal{I}_{-}} (-\lambda_i \otimes \mathbf{x}^{(i)}_j) := \mathcal{S}_j,
\end{equation}

here we use $\mathcal{S}_j$ to denote the tokens occurrences at position $j$, note that 

$$\lambda \mathbf{W}_{V}\mathbf{x}^{(i)}_j \sigma'(\mathbf{x}^{(i)}_j,\mathbf{x}_{n-1}) = \sum_{\mathbf{s} \in (\lambda \otimes \mathbf{x}^{(i)}_j)} \mathbf{W}_{V}\mathbf{s} \sigma'(\mathbf{s},\mathbf{x}_{n-1})$$
$$\lambda \sigma'(\mathbf{x}^{(i)}_j,\mathbf{x}_{n-1}) = \sum_{\mathbf{s} \in (\lambda \otimes \mathbf{x}^{(i)}_j)} \sigma'(\mathbf{s},\mathbf{x}_{n-1})$$

, so we can derive the following equation:

\begin{equation}
\begin{split}
&\sum_{j=0}^{n-1}\lambda_i (\sigma'(\mathbf{x}^{(i)}_j,\mathbf{x}_{n-1}))\\
=&\sum_{i \in \mathcal{I}_{+}} \lambda_i \sigma'(\mathbf{x}^{(i)}_j,\mathbf{x}_{n-1}) - \sum_{i \in \mathcal{I}_{-}} -\lambda_i \sigma'(\mathbf{x}^{(i)}_j,\mathbf{x}_{n-1}) \\
=&\sum_{\mathbf{s} \in \mathcal{S}_j} \sigma'(\mathbf{s},\mathbf{x}_{n-1}) - \sum_{\mathbf{s} \in \mathcal{S}_j} \sigma'(\mathbf{s},\mathbf{x}_{n-1})\\
=&\mathbf{0} 
\end{split},
\end{equation}

\begin{equation}
\begin{split}
&\sum_{j=0}^{n-1}\lambda_i (\mathbf{W}_{V}\mathbf{x}^{(i)}_j \sigma'(\mathbf{x}^{(i)}_j,\mathbf{x}_{n-1}))\\
=&\sum_{i \in \mathcal{I}_{+}} \lambda_i (\mathbf{W}_{V}\mathbf{x}^{(i)}_j \sigma'(\mathbf{x}^{(i)}_j,\mathbf{x}_{n-1})) - \sum_{i \in \mathcal{I}_{-}} -\lambda_i (\mathbf{W}_{V}\mathbf{x}^{(i)}_j \sigma'(\mathbf{x}^{(i)}_j,\mathbf{x}_{n-1})) \\
=&\sum_{\mathbf{s} \in \mathcal{S}_j} (\mathbf{W}_{V}\mathbf{s} \sigma'(\mathbf{s},\mathbf{x}_{n-1})) - \sum_{\mathbf{s} \in \mathcal{S}_j} (\mathbf{W}_{V}\mathbf{s}  \sigma'(\mathbf{s},\mathbf{x}_{n-1}))\\
=&\mathbf{0} 
\end{split},
\end{equation}
\end{proof}

Based on \cref{prop:out_linear_softmax}, we observe that the single layer single head softmax transformer shares a similar dependent property with the ReLU attention only transformer:

\begin{proposition}
\label{prop:pre_softmax}
If the input sequences $\mathbf{X}^{(0)},\mathbf{X}^{(1)},\dots, \mathbf{X}^{(N-1)} \in \mathcal{X}^n$  are \textbf{dependent}, then for any single layer single head softmax transformer $\texttt{TF}_{\text{soft}}$, their prediction result $\mathbf{o}^{(0)},\dots,\mathbf{o}^{(N-1)}$
 $$
 \lambda'_0 \mathbf{o}^{(0)}  + \lambda'_1 \mathbf{o}^{(1)} + \dots + \lambda'_{N-1} \mathbf{o}^{(N-1)} = \mathbf{0}.
 $$
 where $\lambda'_i = \sum_{j=0}^{n-1}\sigma'(\mathbf{x}^{(i)}_j,\mathbf{x}_{n-1}) \lambda_i$, $\{\lambda_i\}_{i=0}^{N-1}$ represents the coefficients defined in \cref{def:data_dependence}. 

\end{proposition}

Expanding our results from \cref{prop:lsa_fail},\cref{prop:tmplt_fail} to the softmax attention only transformer, we can conclude that a single-layer single-head attention only transformer is incapable of handling our reasoning and generalization tasks.

\end{document}